\documentclass{article}

\usepackage{microtype}
\usepackage{graphicx}
\usepackage{booktabs} %

\usepackage{hyperref}

\usepackage[accepted]{icml2020}

\usepackage[textsize=small,disable]{todonotes}
\usepackage{enumitem}

\usepackage{xcite}

\usepackage{xr}
\makeatletter
\newcommand*{\addFileDependency}[1]{%
  \typeout{(#1)}%
  \@addtofilelist{#1}%
  \IfFileExists{#1}{}{\typeout{No file #1.}}%
}
\makeatother

\newcommand{\mytitle}{Relaxing Bijectivity Constraints with Continuously Indexed Normalising Flows}
\newcommand{\myshorttitle}{Continuously Indexed Flows}

\icmltitlerunning{\myshorttitle}

\usepackage{titlesec}

\usepackage{siunitx}
\usepackage{amsmath}
\usepackage{amssymb}
\usepackage{amsthm}
\usepackage{mathtools}
\usepackage{mathrsfs}
\usepackage{bbm}
\usepackage{booktabs}
\usepackage{multirow}
\usepackage{tabularx}
\usepackage{upgreek}
\usepackage{subcaption}
\usepackage{thm-restate}

\usepackage{tikz}
\usetikzlibrary{bayesnet}
\usetikzlibrary{arrows}
\usepackage{color}
\usepackage{graphicx}
\usepackage{caption}
\usepackage{subcaption}
\usetikzlibrary{backgrounds}

\newtheorem{theorem}{Theorem}[section]

\newtheorem{lemma}[theorem]{Lemma}
\newtheorem{proposition}[theorem]{Proposition}
\newtheorem{corollary}[theorem]{Corollary}

\newcommand{\R}{\mathbb{R}}

\newcommand{\abs}[1]{|#1|}
\newcommand{\Abs}[1]{\left|#1\right|}
\newcommand{\norm}[1]{\lVert#1\rVert}

\DeclareMathOperator{\diag}{diag}

\newcommand{\pushforward}[2]{#1\##2}
\DeclareMathOperator{\LogSumExp}{LogSumExp}

\DeclareMathOperator{\Lip}{Lip}
\DeclareMathOperator{\BiLip}{BiLip}
\newcommand{\closure}[1]{\overline{#1}}

\newcommand{\E}{\mathbb{E}}
\newcommand{\Var}{\mathrm{Var}}

\renewcommand{\P}{\mathbb{P}}

\newcommand{\Dto}{\overset{\mathcal{D}}{\to}}

\newcommand{\ind}{\mathbb{I}}

\newcommand{\Normal}{\mathrm{Normal}}

\DeclarePairedDelimiterX{\infdivx}[2]{(}{)}{#1\;\delimsize\|\;#2}
\newcommand{\KL}{D_{\mathrm{KL}}\infdivx*}

\newcommand{\opnorm}[1]{\norm{#1}_{\mathrm{op}}}

\newcommand{\noisespace}{\mathcal{Z}}
\newcommand{\dataspace}{\mathcal{X}}
\newcommand{\spacedim}{d}

\newcommand{\noise}{z}
\newcommand{\noisevar}{Z}
\newcommand{\data}{x}
\newcommand{\datavar}{X}
\newcommand{\fwdmap}{f}
\newcommand{\Fwdmap}{F}

\newcommand{\target}{p^\star_\datavar}
\newcommand{\model}{p_\datavar}
\newcommand{\prior}{p_\noisevar}
\newcommand{\priormeas}{P_\noisevar}
\newcommand{\targetmeas}{P_\datavar^\star}
\newcommand{\modelmeas}{P_\datavar}

\newcommand{\idx}{u}
\newcommand{\idxspace}{\mathcal{U}}
\newcommand{\idxvar}{U}
\newcommand{\idxdist}{p_{\idxvar|\noisevar}}
\newcommand{\idxcond}{p_{\idxvar|\noisevar}}
\newcommand{\idxcondmeas}{P_{\idxvar|\noisevar}}
\newcommand{\idxmarg}{p_\idxvar}

\newcommand{\fulljoint}{p_{\datavar, \idxvar, \noisevar}}
\newcommand{\dataidxjoint}{p_{\datavar, \idxvar}}
\newcommand{\idxpost}{p_{\idxvar|\datavar}}
\newcommand{\approxidxpost}{q_{\idxvar|\datavar}}
\newcommand{\elbo}{\mathcal{L}}
\newcommand{\scale}{s}
\newcommand{\trans}{t}
\newcommand{\mean}{\mu}

\newcommand{\numlayers}{L}
\newcommand{\layeridx}{\ell}
\newcommand{\lipconst}{M}
\newcommand{\reparamfunc}{H}

\newcommand{\jac}{\mathrm D}

\DeclareMathOperator{\supp}{supp}
\newcommand{\dirac}{\delta}
\newcommand{\idmat}{I}
\DeclareMathOperator{\Id}{Id}
\newcommand*\diff{\mathop{}\!\mathrm{d}}

\newcommand{\circled}[1]{\raisebox{.5pt}{\textcircled{\raisebox{-.9pt} {#1}}}}

\DeclareMathOperator{\interior}{int}
\newcommand{\boundary}{\partial}

\newcommand{\NN}{\texttt{NN}}

\newcommand{\modelname}{continuously indexed flow}
\newcommand{\modelacr}{CIF}

\begin{document}

\twocolumn[
    \icmltitle{\mytitle}

\begin{icmlauthorlist}
\icmlauthor{Rob Cornish}{ox}
\icmlauthor{Anthony Caterini}{ox}
\icmlauthor{George Deligiannidis}{ox,ati}
\icmlauthor{Arnaud Doucet}{ox}
\end{icmlauthorlist}

\icmlaffiliation{ox}{University of Oxford, Oxford, United Kingdom}
\icmlaffiliation{ati}{The Alan Turing Institute, London, United Kingdom}
\icmlcorrespondingauthor{Rob Cornish}{rcornish@robots.ox.ac.uk}

\icmlkeywords{Machine Learning, ICML, Normalizing Flows, Normalising Flows, Density Estimation, Topology, Invertible Neural Networks, Lipschitz continuity, Lipschitz networks}

    \vskip 0.3in
]

\printAffiliationsAndNotice{}

\begin{abstract}
We show that normalising flows become pathological when used to model targets whose supports have complicated topologies.
In this scenario, we prove that a flow must become arbitrarily numerically noninvertible in order to approximate the target closely.
This result has implications for all flow-based models, and especially \emph{residual flows} (ResFlows), which explicitly control the Lipschitz constant of the bijection used.
To address this, we propose \emph{\modelname s} (\modelacr s), which replace the single bijection used by normalising flows with a continuously indexed family of bijections, and which can intuitively ``clean up'' mass that would otherwise be misplaced by a single bijection.
We show theoretically that \modelacr s are not subject to the same topological limitations as normalising flows, and obtain better empirical performance on a variety of models and benchmarks.

\end{abstract}
\section{Introduction}

\emph{Normalising flows} \citep{rezende2015variational} have become popular methods for density estimation \citep{dinh2016density,papamakarios2017masked,kingma2018glow,chen2019residual}.
These methods model an unknown target distribution $\targetmeas$ on a data space $\dataspace \subseteq \R^{\spacedim}$ as the marginal of $\datavar$ obtained by the generative process
\begin{equation} \label{eq:nf_generative}
    \noisevar \sim \priormeas,  \quad 
    \datavar \coloneqq \fwdmap(\noisevar),
\end{equation}
where $\priormeas$ is a \emph{prior} distribution on a space $\noisespace \subseteq \R^d$, and $\fwdmap : \dataspace \to \noisespace$ is a bijection. The use of a bijection means the density of $\datavar$ can be computed analytically by the change-of-variables formula, and the parameters of $\fwdmap$ can be learned by maximum likelihood using i.i.d.\ samples from $\targetmeas$.

\begin{figure}
    \centering

    \begin{subfigure}{.3\linewidth}
        \centering
        \includegraphics[width=\textwidth]{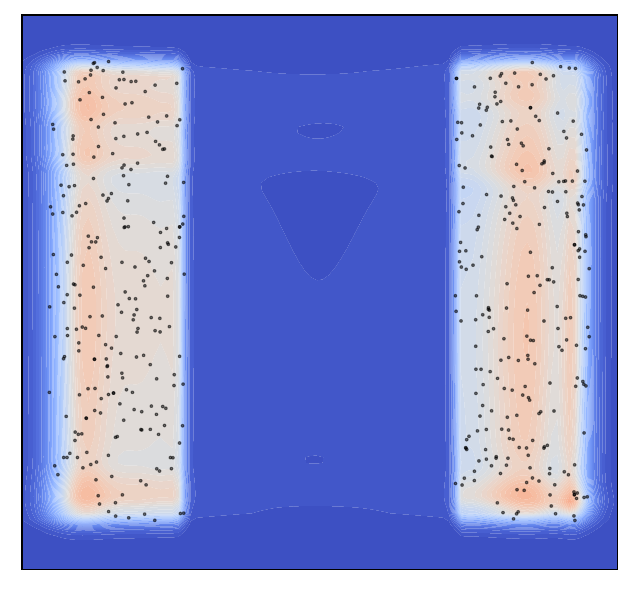}
    \end{subfigure}%
    \begin{subfigure}{.3\linewidth}
      \centering
        \includegraphics[width=\textwidth]{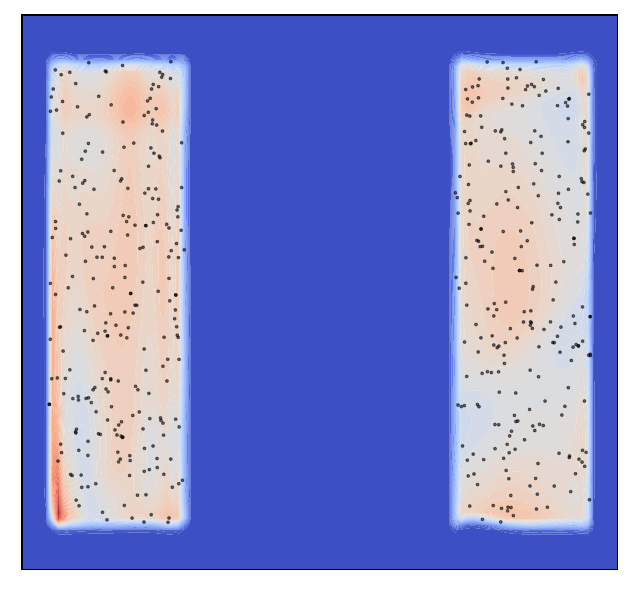}
    \end{subfigure}%
    \begin{subfigure}{.3\linewidth}
      \centering
        \includegraphics[width=\textwidth]{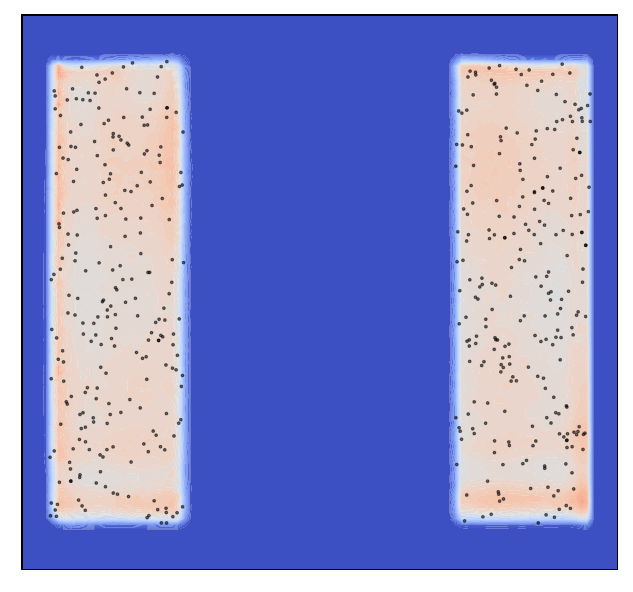}
    \end{subfigure}
    
    \begin{subfigure}{.3\linewidth}
        \centering
        \includegraphics[width=\textwidth]{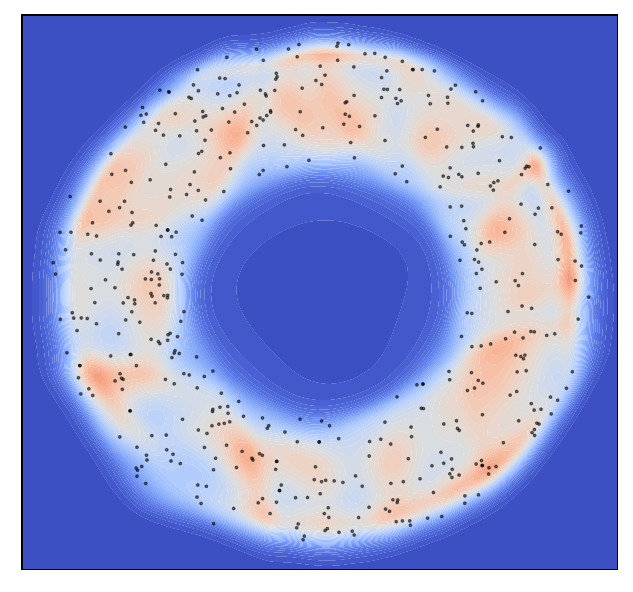}
    \end{subfigure}%
    \begin{subfigure}{.3\linewidth}
      \centering
        \includegraphics[width=\textwidth]{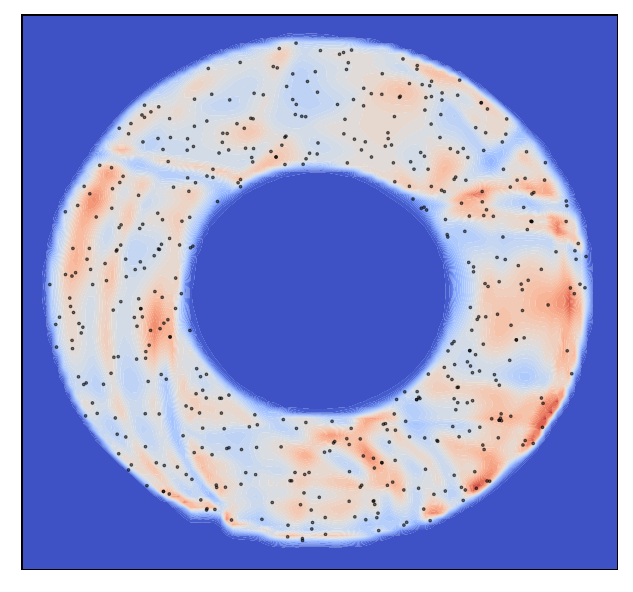}
    \end{subfigure}%
    \begin{subfigure}{.3\linewidth}
      \centering
        \includegraphics[width=\textwidth]{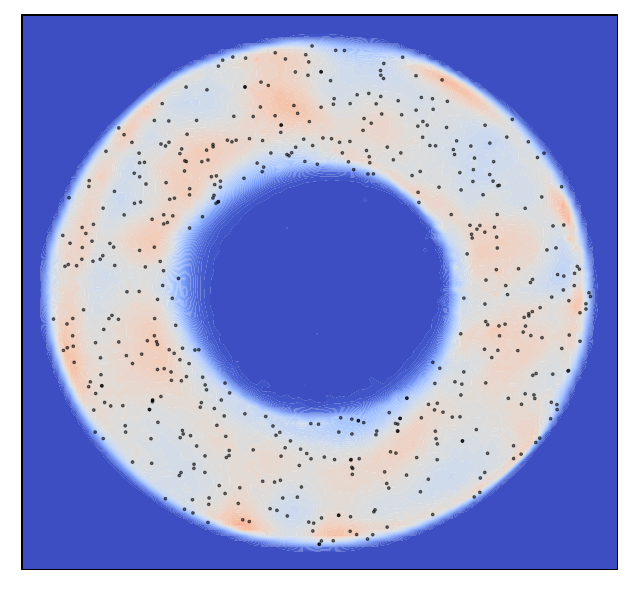}
    \end{subfigure}
    
    \caption{Densities learned by a 10-layer ResFlow (left), 100-layer ResFlow (middle), and 10-layer \modelacr-ResFlow (right) for two datasets (samples shown in black) that are not homeomorphic to the Gaussian prior. The 10-layer ResFlow visibly leaks mass outside of the support of the target due to its small bi-Lipschitz constant. The larger ResFlow improves on this, but still achieves smaller average log probability than the \modelacr-ResFlow, as is apparent from the greater homogeneity of the right-hand densities.}
    \label{fig:iResNet-2uniforms-annulus}
\end{figure}

To be effective, a normalising flow model must specify an expressive family of bijections with tractable Jacobians.
Affine coupling layers \citep{dinh2014nice, dinh2016density}, autoregressive maps \citep{germain2015made,papamakarios2017masked}, invertible linear transformations \citep{kingma2018glow}, ODE-based maps \citep{grathwohl2018ffjord}, and invertible ResNet blocks \citep{behrmann2019invertible,chen2019residual} are all examples of such bijections that can be composed to produce expressive flows.
These models have demonstrated significant promise in their ability to model complex datasets and to synthesise realistic data.

In all these cases, $\fwdmap$ and $\fwdmap^{-1}$ are both continuous.
It follows that $\fwdmap$ is a \emph{homeomorphism}, and therefore preserves the topology of its domain \citep[Definition 3.3.10]{runde2007taste}.
As \citet{dupont2019augmented} and \citet{dinh2019rad} mention, this seems intuitively problematic when $\priormeas$ and $\targetmeas$ are supported on domains with distinct topologies, which occurs for example when the supports differ in their number of connected components or ``holes'', or when they are ``knotted'' differently.
This seems inevitable in practice, as $\priormeas$ is usually quite simple (e.g.\ a Gaussian) while $\targetmeas$ is very complicated (e.g.\ a distribution over images). 

As our first contribution, we make precise the consequences of using a topologically misspecified prior.
We confirm that in this case it is indeed impossible to recover the target perfectly if $\fwdmap$ is a homeomorphism.
Moreover, in \autoref{thm:exploding-bilipschitz-constant} we prove that, in order to \emph{approximate} such a target arbitrarily well, we must have $\BiLip \fwdmap \to \infty$, where $\BiLip \fwdmap$ denotes the \emph{bi-Lipschitz constant} of $\fwdmap$ defined as the infimum over $\lipconst \in [1, \infty]$ such that
\begin{equation} \label{eq:bilipschitz-defn}
     \lipconst^{-1} \norm{\noise - \noise'} \leq \norm{\fwdmap(\noise) - \fwdmap(\noise')} \leq \lipconst \norm{\noise - \noise'}
\end{equation}
for all $\noise, \noise' \in \noisespace$.
\autoref{thm:exploding-bilipschitz-constant} applies essentially regardless of the training objective, and has implications for the case that $\priormeas$ and $\targetmeas$ both have full support but are heavily concentrated on regions that are not homeomorphic.
Since $\BiLip \fwdmap$ is a natural measure of the ``invertibility'' of $\fwdmap$ \citep{behrmann2020on}, this result shows that the goal of designing neural networks with well-conditioned inverses is fundamentally at odds with the goal of designing neural networks that can approximate complicated densities.

\autoref{thm:exploding-bilipschitz-constant} also has immediate implications for \emph{residual flows} (ResFlows) \citep{behrmann2019invertible,chen2019residual}, which have recently achieved state-of-the-art performance on several large-scale density estimation tasks.
Unlike models based on triangular maps \citep{jaini2019sum}, ResFlows have the attractive feature that the structure of their Jacobians is unconstrained, which may explain their greater expressiveness.
However, as part of the construction, the bi-Lipschitz constant of $\fwdmap$ is bounded, and so these models must be composed many times in order to achieve overall the large bi-Lipschitz constant required for a complex $\targetmeas$.\footnote{\citet{chen2019residual} report using 100-200 layers to learn even simple 2D densities.}

To address this problem we introduce \emph{\modelname s} (\modelacr s), which generalise \eqref{eq:nf_generative} by replacing the single bijection $\fwdmap$ with an indexed family of bijections $\{\Fwdmap(\cdot; \idx)\}_{\idx \in \idxspace}$, where the index set $\idxspace$ is continuous.
Intuitively, \modelacr s allow mass that would be erroneously placed by a single bijection to be rerouted into a more optimal location.
We show that \modelacr s can learn the support of a given $\targetmeas$ exactly regardless of the topology of the prior, and without the bi-Lipschitz constant of any $\Fwdmap(\cdot;\idx)$ necessarily becoming infinite.
\modelacr s do not specify the form of $\Fwdmap$, and can be used in conjunction with any standard normalising flow architecture directly.

Our use of a continuous index overcomes several limitations associated with alternative approaches based on a discrete index \citep{dinh2019rad,duan2019transport}, which suffer either from a discontinuous loss landscape or an intractable computational complexity.
However, as a consequence, we sacrifice the ability to compute the likelihood of our model analytically.
To address this, we propose a variational approximation that exploits the bijective structure of the model and is suitable for training large-scale models in practice.
We empirically evaluate \modelacr s applied to ResFlows, neural spline flows (NSFs) \citep{durkan2019neural}, masked autoregressive flows (MAFs) \citep{papamakarios2017masked}, and RealNVPs \citep{dinh2016density}, obtaining improved performance in all cases. 
We observe a particular benefit for ResFlows: with a 10-layer \modelacr-ResFlow we surpass the performance of a 100-layer baseline ResFlow and achieve state-of-the-art results on several benchmark datasets.

\section{Bi-Lipschitz Constraints on Pushforwards} \label{sec:limitations}

Normalising flows fall into a larger class of density estimators based on \emph{pushforwards}. Given a prior measure $\priormeas$ on $\noisespace$ and a mapping $\fwdmap : \noisespace \to \dataspace$, these models are defined as
\[
    \modelmeas \coloneqq \pushforward{\fwdmap}{\priormeas},
\]
where the right-hand side denotes a distribution with $\pushforward{\fwdmap}{\priormeas}(B) \coloneqq \priormeas(\fwdmap^{-1}(B))$ for Borel $B \subseteq \dataspace$. 
Normalising flows take $\fwdmap$ to be bijective, which under sufficient regularity yields a closed-form expression for the density\footnote{Throughout, by ``density'' we mean Lebesgue density. We will write densities using lowercase, e.g.\ $\model$ for the measure $\modelmeas$.}  of $\modelmeas$ \citep[Theorem 17.2]{billingsley2008probability}. 

Intuitively, the pushforward map $\fwdmap$ \emph{transports} the mass allocated by $\priormeas$ into $\dataspace$-space, thereby defining $\modelmeas$ based on where each unit of mass ends up.
This imposes a global constraint on $\fwdmap$ if $\modelmeas$ is to match perfectly a given target $\targetmeas$.
In particular, denote by $\supp \priormeas$ the \emph{support} of $\priormeas$.
While the precise definition of the support involves topological formalities (see \autoref{sec:topological-preliminaries} in the Supplement), intuitively this set defines the region of $\noisespace$ to which $\priormeas$ assigns mass.
It is then straightforward to show that $\modelmeas = \targetmeas$ only if
\begin{equation} \label{eq:pushforward-support-constraint}
    \supp \targetmeas = \closure{\fwdmap(\supp \priormeas)},
\end{equation}
where $\closure{A}$ denotes the closure of $A$ in $\dataspace$.\footnote{See \autoref{lem:pushforward-support} in the Supplement for a proof.}

The constraint \eqref{eq:pushforward-support-constraint} is especially onerous for normalising flows because of their bijectivity.
In practice, $\fwdmap$ and $\fwdmap^{-1}$ are invariably both continuous, and so $\fwdmap$ is a \emph{homeomorphism}.
Consequently, for these models \eqref{eq:pushforward-support-constraint} entails\footnote{Note that $\closure{\fwdmap(\supp\priormeas)} = \fwdmap(\supp\priormeas)$ here since $\supp \priormeas$ is closed by \autoref{lem:support-closed} in the Supplement.}
\begin{equation} \label{eq:homeomorphic-necessary-condition}
    \text{$\supp \modelmeas = \supp \targetmeas$ only if $\supp \priormeas \cong \supp \targetmeas$,}
\end{equation}
where $\mathcal{A} \cong \mathcal{B}$ means that $\mathcal{A}$ and $\mathcal{B}$ are \emph{homeomorphic}, i.e.\ isomorphic as topological spaces \citep[Definition 3.3.10]{runde2007taste}.
This means that $\supp \priormeas$ and $\supp \targetmeas$ must exactly share \emph{all} topological properties, including number of connected components, number of ``holes'', the way they are ``knotted'', etc., in order to learn the target perfectly. 
Condition \eqref{eq:homeomorphic-necessary-condition} therefore suggests that normalising flows are not optimally suited to the task of learning complex real-world densities, where such topological mismatch seems inevitable. 

However, \eqref{eq:homeomorphic-necessary-condition} only rules out the limiting case $\modelmeas = \targetmeas$.
In practice it is likely enough to have $\modelmeas \approx \targetmeas$, and it is therefore relevant to consider the implications of a topologically misspecified prior in this case also.
Intuitively, this seems to require $\fwdmap$ become \emph{almost} nonbijective as $\modelmeas$ approaches $\targetmeas$, but it is not immediately clear what this means, or whether this must occur for all models.
Likewise, in practice it might be reasonable to assume the density of $\targetmeas$ is everywhere strictly positive.
In this case, even if $\targetmeas$ is \emph{concentrated} on some very complicated set, the constraint \eqref{eq:homeomorphic-necessary-condition} would trivially be met if $\priormeas$ is Gaussian, for example.
Nevertheless, it seems that infinitesimal regions of mass should not significantly change the behaviour required of $\fwdmap$, and we would therefore like to extend \eqref{eq:homeomorphic-necessary-condition} to apply here also.

The bi-Lipschitz constant \eqref{eq:bilipschitz-defn} naturally quantifies the ``invertibility'' of $\fwdmap$.
\citet{behrmann2020on} recently showed a relationship between the bi-Lipschitz constant and the \emph{numerical} invertibility of $\fwdmap$.
If $\fwdmap$ is injective and differentiable,
\[
    \BiLip \fwdmap = \max\left(\sup_{\noise \in \noisespace} \opnorm{\jac \fwdmap(\noise)}, \sup_{\data \in \fwdmap(\noisespace)} \opnorm{\jac \fwdmap^{-1}(\data)} \right),
\]
where $\jac g(y)$ is the Jacobian of $g$ at $y$ and $\opnorm{\cdot}$ is the operator norm.
A large bi-Lipschitz constant thus means $\fwdmap$ or $\fwdmap^{-1}$ ``jumps'' somewhere in its domain.
More generally, if $\fwdmap$ is not injective, then $\BiLip \fwdmap = \infty$, while if $\BiLip \fwdmap < \infty$, then $\fwdmap$ is a homeomorphism from $\noisespace$ to $\fwdmap(\noisespace)$.\footnote{See \autoref{sec:lip-and-bilip} in the Supplement for proofs.}

The following theorem shows that if the supports of $\priormeas$ and $\targetmeas$ are not homeomorphic, then the bi-Lipschitz constant of $\fwdmap$ must grow arbitrarily large in order to approximate $\targetmeas$.
Here $\Dto$ denotes weak convergence.
\begin{restatable}{theorem}{lipthm}\label{thm:exploding-bilipschitz-constant}
    Suppose $\priormeas$ and $\targetmeas$ are probability measures on $\R^{d_\noisespace}$ and $\R^{d_\dataspace}$ respectively, and that $\supp \priormeas \not \cong \supp \targetmeas$. Then for any sequence of measurable $\fwdmap_n : \R^{d_\noisespace}\to \R^{d_\dataspace}$, we can have $\fwdmap_n \# \priormeas \Dto \targetmeas$ only if
    \[
        \lim_{n \to \infty} \BiLip \fwdmap_n = \infty.
    \]
\end{restatable}
Weak convergence is implied by the minimisation of all standard statistical divergences used to train generative models, including the KL and Jensen-Shannon divergences and the Wasserstein metric \citep[Theorem 2]{arjovsky2017wasserstein}.
Thus, \autoref{thm:exploding-bilipschitz-constant} states that these quantities can vanish only if the bi-Lipschitz constant of the learned mapping becomes arbitrarily large.
Likewise, note that we do not assume $d_\noisespace = d_\dataspace$ so that this result also applies to injective flow models \citep{kumar2019learning}, as well as other pushforward-based models such as GANs \citep{goodfellow2014generative}.\footnote{However, the implications for GANs seem less problematic since a GAN generator is not usually assumed to be bijective.}

\autoref{thm:exploding-bilipschitz-constant} also applies when $\supp \priormeas$ is \emph{almost} not homeomorphic to $\supp \targetmeas$, as is made precise by the following corollary. Here $\rho$ denotes any metric for the weak topology; see Chapter 6 of \citet{villani2008optimal} for standard examples.
\begin{restatable}{corollary}{lipthmcorr}
    Suppose $\priormeas$ and $P_\datavar^0$ are probability measures on $\R^{d_\noisespace}$ and $\R^{d_\dataspace}$ respectively with $\supp \priormeas \not \cong \supp P_X^0 $.
    Then there exists nonincreasing $\lipconst : [0, \infty) \to [1, \infty]$ with $\lipconst(\epsilon) \to \infty$ as $\epsilon \to 0$ such that, for any probability measure $\targetmeas$ on $\R^{d_\dataspace}$, we have $\BiLip \fwdmap \geq \lipconst(\epsilon)$ whenever  $\rho(\targetmeas, P_\datavar^0) \leq \epsilon$ and %
    $\rho(\pushforward{\fwdmap}{\priormeas}, \targetmeas) \leq \epsilon$.
\end{restatable}
In other words, if the target is close to a probability measure with non-homeomorphic support to that of the prior (i.e.\ $\rho(\targetmeas, P_\datavar^0)$ is small), and if the model is a good approximation of the target (i.e.\ $\rho(\pushforward{\fwdmap}{\priormeas}, \targetmeas)$ is small), then the Bi-Lipschitz constant of $\fwdmap$ must be large.

Proofs of these results are in \autoref{sec:exploding-bilipschitz-constant-proof} of the Supplement.

\subsection{Practical Implications}

The results of this section indicate a limitation of existing flow-based density models.
This is most direct for \emph{residual flows} (ResFlows) \cite{behrmann2019invertible,chen2019residual}, which take $\fwdmap = \fwdmap_\numlayers \circ \cdots \circ \fwdmap_1$ with each layer of the form
\begin{equation} \label{eq:resflow}
    \fwdmap_\layeridx^{-1}(\data) = \data + g_\layeridx(\data), \qquad \Lip g_\layeridx \leq \kappa < 1.
\end{equation}
Here $\Lip$ denotes the Lipschitz constant, which is bounded by a fixed constant $\kappa$ throughout training.
The Lipschitz constraint is enforced by spectral normalisation \citep{miyato2018spectral,gouk2018regularisation} and ensures each $\fwdmap_\layeridx$ is bijective.
However, it also follows \citep[Lemma 2]{behrmann2019invertible} that
\begin{equation} \label{eq:bilip-constraint}
    \BiLip \fwdmap \leq \max(1 + \kappa, (1 - \kappa)^{-1})^\numlayers < \infty,
\end{equation}
and \autoref{thm:exploding-bilipschitz-constant} thus restricts how well a ResFlow can approximate $\targetmeas$ with non-homeomorphic support to $\priormeas$.
\autoref{fig:iResNet-2uniforms-annulus} illustrates this in practice for simple 2-D examples.

It is possible to relax \eqref{eq:bilip-constraint} by taking $\kappa \to 1$.
However, this can have a detrimental effect on the variance of the Russian roulette estimator \citep{kahn1955use} used by \citet{chen2019residual} to compute the Jacobian, and in \autoref{sec:russian-roulette-variance} of the Supplement we give a simple example in which the variance is in fact infinite.
Alternatively, we can also loosen the bound \eqref{eq:bilip-constraint} by taking $\numlayers \to \infty$, and \autoref{fig:iResNet-2uniforms-annulus} shows that this does indeed lead to better performance.
However, greater depth means greater computational cost.
In the next section we describe an alternative approach that allows relaxing the bi-Lipschitz constraint of \autoref{thm:exploding-bilipschitz-constant} without modifying either $\kappa$ or $\numlayers$, and thus avoids these potential issues.

Unlike ResFlows, most normalising flows used in practice have an unconstrained bi-Lipschitz constant \citep{behrmann2020on}.
As as result, \autoref{thm:exploding-bilipschitz-constant} does not prevent these models from approximating non-homeomorphic targets arbitrarily well, and indeed several architectures have been proposed that can in principle do so \citep{huang2018neural,jaini2019sum}.
Nevertheless, the constraint \eqref{eq:homeomorphic-necessary-condition} shows that these models still face an underlying limitation in practice, and suggests we may improve performance more generally by relaxing the requirement of bijectivity.
We verify empirically in \autoref{sec:experiments} that, in addition to ResFlows, our proposed method also yields benefits for flows without an explicit bi-Lipschitz constraint.

Finally, \autoref{thm:exploding-bilipschitz-constant} has implications for the numerical stability of normalising flows.
It was recently pointed out by \citet{behrmann2020on} that, while having a well-defined mathematical inverse, many common flows can become \emph{numerically} noninvertible over the course of training, leading to low-quality reconstructions and calling into question the accuracy of density values output by the change-of-variables formula.
\citet{behrmann2020on} suggest explicitly constraining $\BiLip \fwdmap$ in order to avoid this problem.
\autoref{thm:exploding-bilipschitz-constant} shows that this involves a fundamental tradeoff against expressivity: if greater numerical stability is required of our normalising flow, then we must necessarily reduce the set of targets we can represent arbitrarily well.

\section{\myshorttitle} \label{sec:model}

In this section we propose \emph{\modelname s} (\modelacr s) for relaxing the bijectivity of standard normalising flows. We begin by defining the model we consider, and then detail our suggested training and inference procedures. In the next section we discuss advantages over related approaches.

\subsection{Model Specification}

\modelacr s are obtained by replacing the single bijection $\fwdmap$ used by normalising flows with an indexed family $\{\Fwdmap(\cdot; \idx)\}_{\idx \in \idxspace}$, where $\idxspace \subseteq \R^{d_\idxspace}$ is our index set and each $\Fwdmap(\cdot; \idx) : \noisespace \to \dataspace$ is a bijection.
We then define the model $\modelmeas$ as the marginal of $\datavar$ obtained from the following generative process:
\begin{equation} \label{eq:our-model}
    \noisevar \sim \priormeas, \quad \idxvar \sim \idxcondmeas(\cdot | \noisevar),   \quad \datavar \coloneqq \Fwdmap(\noisevar; \idxvar).
\end{equation}
Like \eqref{eq:nf_generative}, we assume a prior $\priormeas$ on $\noisespace$, but now also require conditional distributions $\idxcondmeas(\cdot|\noise)$ on $\idxspace$ for each $\noise \in \noisespace$. 

We can increase the complexity of \eqref{eq:our-model} by taking $\priormeas$ itself to have the same form.
This is directly analogous to the standard practice of composing simple bijections to obtain a richer class of normalising flows.
In our context, stacking $\numlayers$ layers of \eqref{eq:our-model} corresponds to the generative process
\begin{equation} \label{eq:our-model-stacked}
    \noisevar_0 \hspace{-.14em} \sim \hspace{-.14em} P_{\noisevar_0}, \hspace{.22em}
    \idxvar_\layeridx \hspace{-.14em} \sim \hspace{-.14em} P_{\idxvar_\layeridx | \noisevar_{\layeridx\!-\!1}}\!(\cdot | \noisevar_{\layeridx\!-\!1}), \hspace{.22em}
    \noisevar_\layeridx \hspace{-.14em} \coloneqq \hspace{-.14em} \Fwdmap_\layeridx(\noisevar_{\layeridx\!-\!1}; \idxvar_\layeridx),\hspace{-.1em}
\end{equation}
where $\layeridx \in \{1, \ldots, \numlayers\}$.
We then take $\modelmeas$ to be the marginal of $\datavar \coloneqq \noisevar_\numlayers$.
We have found this construction to improve significantly the expressiveness of our models and make extensive use of it in our experiments below.
Note that this corresponds to an instance of \eqref{eq:our-model} where, defining $\Fwdmap^\layeridx(\cdot ; \idx_1, \ldots, \idx_\ell) \coloneqq \Fwdmap_\layeridx(\cdot; \idx_\layeridx) \circ \cdots \circ \Fwdmap_1(\cdot; \idx_1)$, we take $\noisevar = \noisevar_0$, $\idxvar = (\idxvar_1, \ldots, \idxvar_\numlayers)$,  $\idxcondmeas(\diff \idx|\noise) = \prod_\ell P_{\idxvar_\layeridx | \noisevar_{\layeridx-1}}(\diff \idx_\layeridx | \Fwdmap^\layeridx(\noise; \idx_1, \ldots, \idx_\layeridx))$, and $\Fwdmap = \Fwdmap^\numlayers$.
We use this to streamline some of the discussion below.

Previous works, most notably RAD \citep{dinh2019rad}, have considered related models with a discrete index set $\idxspace$.
We instead consider a \emph{continuous} index.
In particular, our $\idxspace$ will be an open subset of $\R^{d_\idxspace}$, with each $\idxcondmeas(\cdot | \noise)$ having a density $\idxcond(\cdot|\noise)$.
A continuous index confers various advantages that we describe in \autoref{sec:comparisons}.
The choice also requires a distinct approach to training and inference that we describe in \autoref{sec:training-and-inference}.

We require choices of $\idxcond$ and $\Fwdmap$ for each layer of our model. Straightforward possibilities are
\begin{align}
    \Fwdmap(\noise; \idx) &= \fwdmap\left(e^{-\scale(\idx)} \odot \noise - \trans(\idx)\right) \label{eq:our-bijection-family} \\
    \idxcond(\cdot|\noise) &= \Normal(\mean^p(\noise), \Sigma^p(\noise)) \label{eq:our-idxcond}
\end{align}
for any bijection $\fwdmap$ (e.g.\ a ResFlow step) and appropriately defined neural networks $\scale$, $\trans$, $\mean^p$, and $\Sigma^p$.\footnote{Note this requires $\noisespace = \dataspace = \R^d$ and $\idxspace = \R^{d_\idxspace}$, i.e.\ these domains are not strict subsets. We assume this in all our experiments.}
Here the exponential of a vector is meant elementwise, and $\odot$ denotes elementwise multiplication.
Note that \eqref{eq:our-bijection-family} may be used with all existing normalising flow implementations out-of-the-box.
These choices yielded strong empirical results despite their simplicity, but more sophisticated alternatives are certainly possible and may bring improvements in some applications.

\subsection{Training and Inference} \label{sec:training-and-inference}

Heuristically,\footnote{We make this rigorous in \autoref{sec:density-of-a-cif} of the Supplement.} \eqref{eq:our-model} yields the joint ``density''
\[
    \fulljoint(\data, \idx, \noise) \coloneqq \prior(\noise) \, \idxdist(\idx | \noise) \, \dirac(\data - \Fwdmap(\noise; \idx)),
\]
where $\prior$ is the density of $\priormeas$ and $\dirac$ is the Dirac delta.
If $\Fwdmap$ is sufficiently regular, we can marginalise out the dependence on $\noise$ by making the change of variable $\data' \coloneqq  \Fwdmap(\noise; \idx)$, which means $\diff \noise = \abs{\det \jac \Fwdmap^{-1}(\data'; \idx)} \diff \data'$.\footnote{Here $\jac \Fwdmap(\noise; \idx)$ denotes the Jacobian with respect to $\noise$ only.}
This yields a proper density for $(\datavar, \idxvar)$ by integrating over $\data'$:
\vspace{-.5em}
\begin{multline} \label{eq:cif-joint-density}
    \dataidxjoint(\data, \idx) \coloneqq \prior(\Fwdmap^{-1}(\data; \idx)) \\ \times \idxdist(\idx|\Fwdmap^{-1}(\data; \idx)) \, \abs{\det \jac \Fwdmap^{-1}(\data; \idx)}.
\end{multline}%
For an $\numlayers$-layered model, an extension of this argument also gives the following joint density for each $(\noisevar_\layeridx, \idxvar_{1:\layeridx})$:
\begin{multline} \label{eq:generative_forward_recursion}
    p_{\noisevar_\layeridx, \idxvar_{1:\layeridx}}(\noise_\layeridx, \idx_{1:\layeridx}) \coloneqq
        p_{\noisevar_{\layeridx-1}, \idxvar_{1:\layeridx-1}}(\Fwdmap^{-1}_\layeridx(\noise_{\layeridx}; \idx_{\layeridx}), \idx_{1:\layeridx-1}) \\
        \times p_{\idxvar_\layeridx | \noisevar_{\layeridx-1}}(\idx_\layeridx | \Fwdmap^{-1}_\layeridx(\noise_\layeridx; \idx_\layeridx))
        \abs{\det \jac \Fwdmap^{-1}_\layeridx(\noise_\layeridx; \idx_\layeridx)}.
\end{multline}
Taking $\datavar \coloneqq \noisevar_\numlayers$ as before we obtain $p_{\datavar, \idxvar_{1:\numlayers}}$ and hence a density for $\modelmeas$ via
\begin{equation} \label{eq:our-density}
    \model(\data) \coloneqq \int p_{\datavar, \idxvar_{1:\numlayers}}(\data, \idx_{1:\numlayers}) \, \diff \idx_{1:\numlayers}.
\end{equation}
Since $\idxspace$ is continuous, this is not analytically tractable. To facilitate likelihood-based training and inference, we make use of a variational scheme that we describe now.

Assuming an $\numlayers$-layered model \eqref{eq:our-model-stacked}, we introduce an approximate posterior density $q_{\idxvar_{1:\numlayers}|\datavar} \approx p_{\idxvar_{1:\numlayers}|\datavar}$ and consider the evidence lower bound (ELBO) of $\log \model(\data)$:
\begin{equation} \label{eq:elbo}
    \elbo(\data) \coloneqq \E_{\idx_{1:\numlayers} \sim q_{\idxvar_{1:\numlayers}|\datavar}(\cdot|\data)} \!\left[\log \frac{p_{\datavar, \idxvar_{1:\numlayers}}(\data, \idx_{1:\numlayers})}{q_{\idxvar_{1:\numlayers}|\datavar}(\idx_{1:\numlayers}|\data)}\right].
\end{equation}
It is a standard result that $\elbo(\data) \leq \log \model(\data)$ with equality if and only if $q_{\idxvar_{1:\numlayers}|\datavar}$ is the exact posterior $p_{\idxvar_{1:\numlayers}|\datavar}$. This allows learning an approximation to $\targetmeas$ by maximising $\sum_{i=1}^n \elbo(\data_i)$ jointly in $p_{\datavar, \idxvar_{1:\numlayers}}$ and $q_{\idxvar_{1:\numlayers}|\datavar}$, where we assume a dataset of $n$ i.i.d.\ samples $\data_i \sim \targetmeas$.

We now consider how to parametrise an effective $q_{\idxvar_{1:\numlayers}|\datavar}$.
Standard approaches to designing inference networks for variational autoencoders (VAEs) \citep{kingma2013auto,rezende2014stochastic,rezende2015variational,kingma2016improved}, while mathematically valid, would not exploit the conditional independencies induced by the bijective structure of \eqref{eq:our-model-stacked}. 
We therefore propose a novel inference network that is specifically targeted towards our model, which we compare with existing VAE approaches in \autoref{sec:vae-comparison}.

In particular, our $q_{\idxvar_{1:\numlayers}|\datavar}$ has the following form:
\begin{equation} \label{eq:approx-posterior-form}
    q_{\idxvar_{1:\numlayers}|\datavar}(\idx_{1:\numlayers}|\data) \coloneqq \prod_{\layeridx=1}^\numlayers q_{\idxvar_\layeridx|\noisevar_\layeridx}(\idx_\layeridx | \noise_\layeridx),
\end{equation}
with $\noise_\numlayers \coloneqq \data$ and $\noise_{\layeridx} \coloneqq \Fwdmap^{-1}_{\layeridx+1}(\noise_{\layeridx+1}; \idx_{\layeridx+1})$ for $\layeridx \in \{1, \ldots, \numlayers-1\}$,
and $q_{\idxvar_\layeridx|\noisevar_\layeridx}$ can be any parameterised conditional density.
We show in \autoref{sec:approx-postrior-has-correct-form} of the Supplement that the posterior $p_{\idxvar_{1:\numlayers}|\datavar}$ factors in the same way as \eqref{eq:approx-posterior-form}, so that we do not lose any generality.
Observe also that this scheme shares parameters between $q_{\idxvar_{1:\numlayers}|\datavar}$ and $p_{\datavar,\idxvar_{1:\numlayers}}$ in a natural way, since the same $\Fwdmap_\ell$ are used in both.

We assume each $q_{\idxvar_\layeridx|\noisevar_\layeridx}$ can be suitably reparametrised \citep{kingma2013auto, rezende2014stochastic} so that, for some function $\reparamfunc_\layeridx$ and some density $\eta_\layeridx$ that does not depend on the parameters of $q_{\idxvar_{1:\numlayers}|\noisevar_\layeridx}$ and $p_{\datavar,\idxvar_{1:\numlayers}}$, we have $\reparamfunc_\layeridx(\epsilon_\layeridx, \noise_\layeridx) \sim q_{\idxvar_\layeridx|\noisevar_\layeridx}(\cdot|\noise_\layeridx)$ when $\epsilon_\layeridx \sim \eta_\layeridx$.
We can then obtain unbiased estimates of $\elbo(\data)$ using \autoref{alg:ELBO},
which corresponds to a single-sample approximation to the expectation in \eqref{eq:elbo}.
It is straightforward to see that \autoref{alg:ELBO} has $\Theta(\numlayers)$ complexity.
Differentiating through this procedure allows maximising $\sum_{i=1}^n \elbo(\data_i)$ via stochastic gradient descent.
At test time, we can also estimate $\log \model(\data)$ directly using importance sampling as described by \citet[(40)]{rezende2014stochastic}.
In particular, letting $\hat{\elbo}^{(1)}, \ldots, \hat{\elbo}^{(m)}$ denote the result of separate calls to $\text{ELBO}(\data)$, we have
\begin{equation} \label{eq:IS-estimator}
    \LogSumExp(\hat{\elbo}^{(1)}, \ldots, \hat{\elbo}^{(m)}) - \log m \to \log \model(\data)
\end{equation}
almost surely as $m \to \infty$.

\begin{algorithm}
\caption{Unbiased estimation of $\elbo(\data)$}
\label{alg:ELBO}
\begin{algorithmic}
\FUNCTION{ELBO($\data$)}
    \STATE $\noise_L \gets \data$
    \STATE $\Delta \gets 0$
    \FOR{$\ell=L,\ldots,1$}
        \STATE $\epsilon \sim \eta_\layeridx$
        \STATE $\idx \gets \reparamfunc_\layeridx(\epsilon, \noise_\layeridx)$
        \STATE $\noise_{\layeridx - 1} \gets \Fwdmap^{-1}_\layeridx(\noise_\layeridx; \idx)$
        \STATE $\Delta \gets \Delta + \log p_{\idxvar_\layeridx|\noisevar_{\layeridx-1}}(\idx|\noise_{\layeridx-1}) - \log q_{\idxvar_\layeridx|\noisevar_\layeridx}(\idx | \noise_\layeridx) $
        \STATE $\quad \quad + \log \abs{\det \jac \Fwdmap^{-1}_\layeridx(\noise_\layeridx; \idx)}$
    \ENDFOR
    \STATE {\bf return} $\Delta + \log p_{\noisevar_0}(\noise_0)$
\ENDFUNCTION
\end{algorithmic}
\end{algorithm}

In all our experiments we used
\begin{equation} \label{eq:our-approx-posterior}
    q_{\idxvar_\layeridx|\noisevar_\layeridx}(\cdot|\noise_\layeridx) = \Normal(\mu^q_\layeridx(\noise_\layeridx), \Sigma^q_\layeridx(\noise_\layeridx))
\end{equation}
for appropriate neural networks $\mu^q_\layeridx$ and $\Sigma^q_\layeridx$, which is immediately reparameterisable as described e.g.\ by \citet{kingma2013auto}. We found this gave good enough performance that we did not require alternatives such as IAF \citep{kingma2016improved}, but such options may also be useful.

Finally, \autoref{alg:ELBO} requires an expression for $\log \abs{\det \jac \Fwdmap^{-1}_\layeridx(\noise_\layeridx; \idx_\layeridx)}$.
For \eqref{eq:our-bijection-family} this is
\[
    \log \Abs{\det \jac \fwdmap_\layeridx^{-1}\left(e^{\scale_\ell(\idx_\ell)} \odot \left(\noise_\ell + \trans_\ell(\idx_\ell)\right)\right)} +\sum_{i=1}^{d} [\scale_\layeridx(\idx_\layeridx)]_i,
\]
where $[\data]_i$ denotes the $i^{\text{th}}$ dimension of $\data$.

\section{Comparison with Related Models} \label{sec:comparisons}

\subsection{Comparison with Normalising Flows}

We now compare \modelacr s with normalising flows, and in particular describe how \modelacr s relax the constraints of bijectivity identified in \autoref{sec:limitations}.

\subsubsection{Advantages} \label{sec:advantages-over-nfs}

Observe that \eqref{eq:our-model} generalises normalising flows: if $\Fwdmap(\cdot; \idx)$ does not depend on $\idx$, then we obtain \eqref{eq:nf_generative}.
Moreover, training with the ELBO in this case does not reduce performance compared with training a flow directly, as the following result shows.
Here the components of our model $\Fwdmap_\theta$, $\idxcond^\theta$, and $\approxidxpost^\theta$ are parameterised by $\theta \in \Theta$, and for a given choice of parameters $\theta$ we will denote by $\modelmeas^\theta$ and $\elbo^\theta$ the corresponding distribution and ELBO \eqref{eq:elbo} respectively.
\begin{restatable}{proposition}{cifnoworse}\label{prop:cif-no-worse-than-flow}
    Suppose there exists $\phi \in \Theta$ such that, for some bijection $\fwdmap : \noisespace \to \dataspace$, $\Fwdmap_\phi(\cdot; \idx) = \fwdmap(\cdot)$ for all $\idx \in \idxspace$.
    Likewise, suppose $\idxcond^\phi$ and $\approxidxpost^\phi$ are such that, for some density $r$ on $\idxspace$, $\idxcond^\phi(\cdot|\noise) = \approxidxpost^\phi(\cdot|\data) = r(\cdot)$ for all $\noise \in \noisespace$ and $\data \in \dataspace$.
    If $\E_{\data \sim \targetmeas}[\elbo^\theta(\data)] \geq \E_{\data \sim \targetmeas}[\elbo^\phi(\data)]$, then
    \[
        \KL{\targetmeas}{\modelmeas^\theta} \leq \KL{\targetmeas}{\pushforward{\fwdmap}{\priormeas}}.
    \]
\end{restatable}
Simply stated, in the limit of infinite data, optimising the ELBO will yield at least as performant a model (as measured by the KL) as any normalising flow our model family can express.
The proof is in \autoref{sec:conditions_to_outperform} of the Supplement.
In practice, our choices \eqref{eq:our-bijection-family}, \eqref{eq:our-idxcond}, and \eqref{eq:our-approx-posterior} can easily realise the conditions of \autoref{prop:cif-no-worse-than-flow} by zeroing out the output weights of the neural networks (other than $\fwdmap$) involved.
Thus, for a given $\fwdmap$, we have reason to expect a comparative or better performing model (as measured by average log-likelihood) when trained as a \modelacr\ rather than as a normalising flow.

We expect this will in fact lead to \emph{improved} performance because, intuitively, $\idxcondmeas$ can reroute $\noise$ that would otherwise map outside of $\supp \targetmeas$.
To illustrate, fix $\fwdmap$ in \eqref{eq:our-bijection-family} and choose some $\noise \in \noisespace$.
If $\fwdmap(\noise) \in \supp \targetmeas$, then setting $\Fwdmap(\noise;\idx) = \fwdmap(\noise)$ for all $\idx \in \idxspace$ as described above ensures $\Fwdmap(\noise; \idxvar) \in \supp \targetmeas$ when $\idxvar \sim \idxcondmeas(\cdot|\noise)$.
If conversely $\fwdmap(\noise) \not \in \supp \targetmeas$, then we \emph{still} have $\Fwdmap(\noise; \idxvar) \in \supp \targetmeas$ almost surely if $\idxcondmeas(\cdot|\noise)$ is supported on $\{\idx \in \idxspace : \Fwdmap(\noise; \idx) \in \supp \targetmeas\}$.
Of course, if $\fwdmap$ is too simple, then $\idxcondmeas$ must heuristically become very complex in order to obtain this behaviour.
This would seem to make inference harder, leading to a looser ELBO \eqref{eq:elbo} and thus overall worse performance after training.
We therefore expect \modelacr s to work well for $\fwdmap$ that, like the 10-layer ResFlow in \autoref{fig:iResNet-2uniforms-annulus}, can learn a close approximation to the support of the target but ``leak'' some mass outside of it due to \eqref{eq:homeomorphic-necessary-condition} or \autoref{thm:exploding-bilipschitz-constant}.
A \modelacr\ can then use $\idxcondmeas$ to ``clean up'' these small extraneous regions of mass.

We provide empirical support for this argument in \autoref{sec:experiments}.
We also summarise our discussion above with the following precise result.
Here $\boundary A$ denotes the boundary of a set $A$.
\begin{proposition} \label{prop:support-correction}
    If $\targetmeas(\boundary \supp \targetmeas) = 0$ and $(\noise, \idx) \mapsto \Fwdmap(\noise; \idx)$ is jointly continuous with
    \begin{equation} \label{eq:support-coverage-condition}
        \closure{\Fwdmap(\supp \priormeas \times \idxspace)} \supseteq \supp \targetmeas,
    \end{equation}
    then there exists $\idxcondmeas$ such that $\supp \modelmeas = \supp \targetmeas$ if and only if, for all $\noise \in \supp \priormeas$, there exists $\idx \in \idxspace$ with
    \begin{equation} \label{eq:support-correction-condition}
        \Fwdmap(\noise; \idx) \in \supp \targetmeas.
    \end{equation}
\end{proposition}
The assumptions here are fairly minimal: the boundary condition ensures $\targetmeas$ is not pathological, and if \eqref{eq:support-coverage-condition} does not hold, then $\KL{\targetmeas}{\modelmeas} = \infty$ for \emph{every} $\idxcondmeas$.\footnote{See \autoref{prop:support-absolute-continuity} and  \autoref{lem:pushforward-support} in the Supplement.}
Additionally, the following result gives a sufficient condition under which it is possible to learn the target exactly.
\begin{proposition}
    If $\Fwdmap(\noise; \cdot) : \idxspace \to \dataspace$ is surjective for each $\noise \in \noisespace$, then there exists $\idxcondmeas$ such that $\modelmeas = \targetmeas$.
\end{proposition}
See \autoref{sec:support-correction-result} of the Supplement for proofs.
These results do not require $\supp \priormeas \cong \supp \targetmeas$, thereby showing \modelacr s relax the constraint \eqref{eq:homeomorphic-necessary-condition} for standard normalising flows.

Of course, in practice, our parameterisation \eqref{eq:our-bijection-family} does not necessarily ensure that $\Fwdmap$ will satisfy these conditions, and our parameterisation \eqref{eq:our-idxcond} may not be expressive enough to instantiate the $\idxcondmeas$ that is required.
However, these results show that \modelacr s provide at least a \emph{mechanism} for correcting a topologically misspecified prior.
When $\Fwdmap$ and $\idxcondmeas$ are sufficiently expressive, we can expect that they will learn to approximate these conditions over the course of training if doing so produces a better density estimate.
We therefore anticipate \modelacr s will improve performance for ResFlows, where \autoref{thm:exploding-bilipschitz-constant} applies, and may have benefits more generally, since all flows are ultimately constrained by \eqref{eq:homeomorphic-necessary-condition}.

\subsubsection{Disadvantages}

On the other hand, \modelacr s introduce additional overhead compared with regular normalising flows.
It therefore remains to show we obtain better performance on a fixed computational budget, which requires using a smaller model.
Empirically this holds for the models and datasets we consider in \autoref{sec:experiments}, but there are likely cases where it does not, particularly if the topologies of the target and prior are similar.

Likewise, \modelacr s sacrifice the exactness of normalising flows.
We do not see this as a significant problem for the task of density estimation, since the importance sampling estimator \eqref{eq:IS-estimator} means that at test time we can obtain arbitrary accuracy by taking $m$ to be large.
However, the lack of a closed-form density does limit the use of \modelacr s in some downstream tasks.
In particular, \modelacr s cannot immediately be plugged in to a variational approximation in the manner of \citet{rezende2015variational}, since this requires exact likelihoods.
However, it may be possible to use \modelacr s in the context of an extended-space variational framework along the lines of \citet{agakov2004auxiliary}, and we leave this for future work.

\subsection{Comparison with Discretely Indexed Models}

Similar models to \modelacr s have been proposed that use a discrete index space.
In the context of Bayesian inference, \citet{duan2019transport} proposes a single-layer ($\numlayers = 1$) model consisting of \eqref{eq:our-model} with $\idxspace = \{1, \ldots, I\}$ and $\Fwdmap(\cdot; i) = \fwdmap_i$ for separate normalising flows $\fwdmap_1, \ldots, \fwdmap_I$.
A special case of this framework is given by \emph{deep Gaussian mixture models} \citep{van2014factoring,van2015locally}, which corresponds to using invertible linear transformations for each $\fwdmap_i$.
In this case, \eqref{eq:our-density} becomes a summation that can be computed analytically.
However, this quickly becomes intractable as $\numlayers$ grows larger, since the cost to compute this is seen to be $\Theta(I^\numlayers)$.
Unlike for a continuous $\idx$, this cannot easily be reduced to $\Theta(\numlayers)$ using a variational approximation as in \autoref{sec:training-and-inference}, since a discrete $\approxidxpost$ is not amenable to the reparameterisation trick.
In addition, the use of separate bijections also means that the number of parameters of the model grows as $I$ increases.
In contrast, a continuous index allows a natural mechanism for sharing parameters across different $\Fwdmap(\cdot; \idx)$ as in \eqref{eq:our-bijection-family}.

Prior to \citet{duan2019transport}, \citet{dinh2019rad} proposed RAD as a means to mitigate the $\Theta(I^\numlayers)$ cost of na\"ively stacking discrete layers.
RAD partitions $\dataspace$ into $I$ disjoint subsets $B_1, \ldots, B_I$ and defines bijections $\fwdmap_i : \noisespace \to B_i$ for each $i$.
The model is then taken to be the marginal of $\datavar$ in
\[
    \noisevar \sim \priormeas, \quad \idxvar \sim \idxcondmeas(\cdot|\noisevar), \quad \datavar \coloneqq \fwdmap_\idxvar(\noisevar),
\]
where each $\idxcondmeas(\cdot|\noise)$ is a discrete distribution on $\{1, \ldots, I\}$.
Note that this is not an instance of our model \eqref{eq:our-model}, since we require each $\Fwdmap(\cdot; \idx)$ to be surjective onto $\dataspace$.
The use of partitioning means that \eqref{eq:our-density} is a summation with only a single term, which reduces the cost for $\numlayers$ layers to $\Theta(\numlayers)$.
However, partitioning also makes $\model$ discontinuous.
This leads to a very difficult optimisation problem and \citet{dinh2019rad} only report results for simple 2-D densities.
Additionally, partitioning requires ad-hoc architectural changes to existing normalising flows, and does not directly address the increasing parameter cost as $I$ grows large.

\begin{figure}[t]
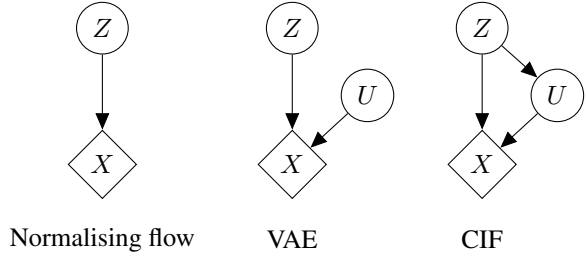

    \centering
  \tikz{
 \node[diamond,draw=black,inner sep=2.5pt] (x1) {$X$};%
 \node[circle,draw=black,above=of x1] (z1) {$Z$};%
 \node[below of=x1] () {Normalising flow};%
 \node[diamond,draw=black,right=of x1,xshift=.6cm,inner sep=2.5pt] (x2) {$X$};%
 \node[circle,draw=black,above=of x2,xshift=1cm,yshift=-.9cm] (u2) {$U$};%
 \node[circle,draw=black,above=of x2] (z2) {$Z$};%
 \node[below of=x2] () {VAE};%
 \node[diamond,draw=black,right=of x2,xshift=.6cm,inner sep=2.5pt] (x3) {$X$};%
 \node[circle,draw=black,above=of x3,xshift=1cm,yshift=-.9cm] (u3) {$U$};%
 \node[circle,draw=black,above=of x3] (z3) {$Z$}; %
 \node[below of=x3] () {\modelacr};%
 \edge {z1} {x1}
 \edge {z2,u2} {x2}
 \edge {z3} {u3}
 \edge {z3,u3} {x3}
}
 \caption{Comparison of related generative models. Circular nodes are random and diamond nodes are deterministic. \modelacr s generalise both normalising flows and VAEs as shown.} \label{fig:model-comparison}
\end{figure}

\subsection{Comparison with Variational Autoencoders} \label{sec:vae-comparison}

\modelacr s also generalise a broad family of variational autoencoders (VAEs) \citep{kingma2013auto,rezende2014stochastic}.
Recall that VAEs take
\begin{equation} \label{eq:vae-marginal}
    \model(\data) \coloneqq \int \idxmarg(\idx) p_{\datavar|\idxvar}(\data|\idx) \diff\idx
\end{equation}
for some choices of densities $\idxmarg$ and $p_{\datavar|\idxvar}$.\footnote{Note that this notation is nonstandard for VAEs in order to align with the rest of the paper. Here our $\idxvar$ corresponds to $z$ as used by \citet{kingma2013auto}.}
For instance, a mean-field Gaussian observation density has
\[
    p_{\datavar|\idxvar}(\cdot|\idx) \coloneqq \Normal\left(\trans(\idx), \diag\left(e^{\scale(\idx)}\right)\right),
\]
where $\trans, \scale : \idxspace \to \dataspace$, and $\diag(v)$ denotes the matrix with diagonal $v \in \R^d$ and zeros elsewhere.
If $\priormeas$ is a standard Gaussian, if each $\idxcondmeas(\cdot|\noise)$ has independent density $\idxmarg$, and if $\Fwdmap$ is \eqref{eq:our-bijection-family} with $\fwdmap$ the identity, then it follows that \eqref{eq:our-model} has marginal density \eqref{eq:vae-marginal} (modulo the signs of $\scale$ and $\trans$).\footnote{Here $\noisevar$ corresponds to $\epsilon$ as used by \citet{kingma2013auto}.}

More generally, every VAE model \eqref{eq:vae-marginal} with each $p_{\datavar|\idxvar}(\cdot|\idx)$  strictly positive corresponds to an instance of \eqref{eq:our-model} where $\idxvar$ is sampled independently of $\noisevar$.
To see this, let $\prior$ be any strictly positive density on $\noisespace$, and let each $\Fwdmap(\cdot; \idx)$ be the Knothe-Rosenblatt coupling \citep{villani2008optimal} of $\prior$ and $p_{\datavar|\idxvar}(\cdot|\idx)$.
By construction each $\Fwdmap(\cdot; \idx)$ is invertible and gives $\Fwdmap(\noisevar; \idx) \sim p_{\datavar|\idxvar}(\cdot|\idx)$ when $\noisevar \sim \prior$.
As a result, \eqref{eq:our-model} again yields $\datavar$ with a marginal density defined by \eqref{eq:vae-marginal}.
Consequently, \modelacr s generalise the VAE framework by adding an additional edge in the graphical model as shown in \autoref{fig:model-comparison}.

On the other hand, \modelacr s differ from VAEs in the way they are composed.
Whereas \modelacr s stack by taking $\prior$ to be a \modelacr, VAEs are typically stacked by taking $\idxmarg$ to be a VAE \citep{rezende2014stochastic,kingma2014semi,burda2015importance,sonderby2016ladder}.
This has implications for the design of the inference network $q_{\idxvar_{1:\numlayers}|\datavar}$.
In particular, a hierarchical VAE obtained in this way is \emph{Markovian}, so that
\[
    p_{\idxvar_{1:\numlayers}|\datavar}(\data, \idx_{1:\numlayers}) = p_{\idxvar_{\numlayers}|\datavar}(\idx_{\numlayers}|\data) \prod_{\layeridx=1}^\numlayers p_{\idxvar_\layeridx|\idxvar_{\layeridx-1}}(\idx_\layeridx|\idx_{\layeridx-1})
\]
where $\numlayers$ is the number of layers.
This directly allows specifying $q_{\idxvar_{1:\numlayers}|\datavar}$ to be of the same form without any loss of generality \citep{kingma2014semi,burda2015importance,sonderby2016ladder}.
Conversely, \modelacr s do not factor in this way, which motivates our alternative approach in \autoref{sec:training-and-inference}.

Note finally that \modelacr s should not be conflated with the large class of methods that use normalising flows to improve the \emph{inference} procedure in VAEs \citep{rezende2015variational,kingma2016improved,berg2018sylvester}.
These approaches are orthogonal to ours and indeed may be useful for improving our own inference procedure by replacing \eqref{eq:our-approx-posterior} with a more expressive model.

\subsection{Other Related Work}

Additional related methods have been proposed.
Within a classification context, \citet{dupont2019augmented} identify topological problems related to ODE-based mappings \citep{chen2018neural}, which like normalising flows are homeomorphisms and hence preserve the topology of their input.
To avoid this, \citet{dupont2019augmented} propose augmenting the data by appending auxiliary dimensions and learning a new mapping on this space.
In contrast, \modelacr s may be understood as augmenting not the data but instead the \emph{model} by considering a family of individual bijections on the \emph{original} space.

In addition, \citet{ho2019flow++} use a variational scheme to improve on the standard dequantisation method proposed by \citet{theis2015note} for modelling image datasets with normalising flows.
This approach is potentially complementary to \modelacr s, but we do not make use of it in our experiments.

\begin{table*}
\caption[Caption Test]{Mean $\pm$ standard error (over 3 seeds) of average test set log-likelihood (in nats). Higher is better. Best performing runs for each group are shown in bold. A $\star$ indicates state-of-the-art performance according to \citet[Table 1]{durkan2019neural}.}
\label{tab:UCI}
\begin{center}
\begin{small}
\begin{sc}
\begin{tabular}{llllll}
\toprule
                              & Power                               & Gas                       & Hepmass                   & Miniboone         & BSDS300 \\
\midrule
ResFlow ($L = 10$)            & $-2.73 \pm 0.03$                    & $4.16 \pm 0.08$           & $-20.68 \pm 0.02$         & $-14.2 \pm 0.10$  & $123.51 \pm 0.09$ \\
ResFlow ($L = 100$)           & $0.48 \pm 0.00$                     & $10.57 \pm 0.17$          & $-16.67 \pm 0.05$         & $-11.16 \pm 0.04$ & $148.05 \pm 0.61$ \\
\modelacr-ResFlow ($L = 10$)  & ${\bf 1.60 \pm 0.21^\star}$         & ${\bf 12.12 \pm 0.10}$    & ${\bf -13.74 \pm 0.03^\star}$         & ${\bf -8.10 \pm 0.04^\star}$ & ${\bf 160.50 \pm 0.08^\star}$ \\
\midrule
MAF                     & $0.19 \pm 0.02$       & $9.23 \pm 0.07$           & $-18.33 \pm 0.10$         & $-10.98 \pm 0.03$ & $156.13 \pm 0.00$  \\
\modelacr-MAF           & ${\bf 0.48 \pm 0.01}$ & ${\bf 12.02 \pm 0.10}$    & ${\bf -16.63 \pm 0.09}$   & ${\bf -9.93 \pm 0.04}$ & ${\bf 156.67 \pm 0.02}$  \\
\midrule
NSF                     & ${\bf 0.69 \pm 0.00}$      & $13.01 \pm 0.02$          &   $-14.30 \pm 0.05$       & $-10.68 \pm 0.06$ & ${\bf 157.59 \pm 0.02 }$ \\
\modelacr-NSF-1      & ${\bf 0.68 \pm 0.01}$       & $12.94 \pm 0.01$          &   ${\bf -13.83 \pm 0.10}$   
  & ${\bf -9.93 \pm 0.06}$    & ${\bf 157.60 \pm 0.02}$   \\
\modelacr-NSF-2      & ${\bf 0.69 \pm 0.00}$      & ${\bf 13.08 \pm 0.00}$          &   $-14.18 \pm 0.09$    &  $-10.80 \pm 0.01$ & $157.56 \pm 0.02$ \\
\bottomrule
\end{tabular}
\end{sc}
\end{small}
\end{center}
\vspace{-1em}
\end{table*}

\section{Experiments} \label{sec:experiments}

We evaluated the performance of \modelacr s on several problems of varying difficulty, including synthetic 2-D data, several tabular datasets, and three image datasets.
In all cases we took $\noisespace = \dataspace = \R^d$ with $d$ the dimension of the dataset.
We used the stacked architecture \eqref{eq:our-model-stacked} with the prior $P_{\noisevar_0}$ a Gaussian.
At each layer, $\Fwdmap$ had form \eqref{eq:our-bijection-family} with $\fwdmap$ a primitive flow step from a baseline architecture (e.g.\ a single residual block for ResFlow).
Each $\idxcond$ and $\approxidxpost$ had form \eqref{eq:our-idxcond} and \eqref{eq:our-approx-posterior} respectively.
We provide an overview of our results for the tabular and image datasets here.
Full experimental details, including additional 2-D figures along the lines of \autoref{fig:iResNet-2uniforms-annulus}, are in \autoref{sec:full-experimental-details} of the Supplement.
See \url{github.com/jrmcornish/cif} for our code.

\subsection{Tabular Datasets} \label{sec:uci-experiments}

We tested the performance of \modelacr s on the tabular datasets used by \citet{papamakarios2017masked}.
For each dataset, we trained 10 and 100-layer baseline fully connected ResFlows, and corresponding 10-layer \modelacr-ResFlows.
The \modelacr-ResFlows had roughly 1.5-4.5\% more parameters (depending on the dimension of the dataset) than the otherwise identical 10-layer ResFlows, and roughly 10\% of the parameters of the 100-layer ResFlows.
\autoref{tab:UCI} reports the average log-probability of the test set that we obtained for each model.
Observe that in all cases \modelacr-ResFlows significantly outperform both baseline models.
Moreover, for all but GAS, the \modelacr-ResFlows achieve state-of-the-art performance based on the results reported by \citet[Table 1]{durkan2019neural}.
This is particularly noticeable for POWER and BSDS300, where \modelacr-ResFlow improves on the best results of \citet{durkan2019neural} by 0.94 and 2.77 nats respectively.

We additionally tried using \emph{masked autoregressive flows} (MAFs) \citep{papamakarios2017masked} and \emph{neural spline flows} (NSFs) \citep{durkan2019neural} for $\fwdmap$.
In each case, we closely match the experimental settings of the baselines and augment using \modelacr s, controlling for the number of parameters used by the \modelacr\ extensions.
\autoref{tab:UCI} reports the average log-probability across the test set for each experiment.
Here, \modelacr-NSF-1 is a \modelacr\ with the same number of parameters as the baseline, and \modelacr-NSF-2 is a model using a baseline configuration for $\fwdmap$ (but having more parameters overall).
We see that \modelacr-MAFs consistently outperform MAFs across datasets; \modelacr-NSFs do not improve upon NSFs as dramatically, although we still notice improvements and would expect to improve further with more hyperparameter tuning.
Lastly it is important to notice that MAFs and NSFs do not restrict the Lipschitz constant of $\fwdmap$.
These results show that \modelacr s can yield benefits for normalising flows even if \autoref{thm:exploding-bilipschitz-constant} is not directly a limitation.

Finally, for ablation purposes we tried taking $\fwdmap$ to be the identity.
We obtained consistently worse performance than for \modelacr-ResFlows and \modelacr-MAF in this case, which aligns with our conjecture in \autoref{sec:advantages-over-nfs} that a performant \modelacr\ requires an expressive base flow $\fwdmap$.
Details and results are given in \autoref{sec:ablating-fwdmap} of the Supplement.

\subsection{Image Datasets}

We also considered \modelacr s applied to the MNIST \cite{lecun1998mnist}, Fashion-MNIST \citep{xiao2017fashion}, and CIFAR-10 \citep{krizhevsky2009learning} datasets.
Following our tabular experiments, we trained a multi-scale convolutional ResFlow and a corresponding \modelacr-ResFlow, as well as a larger baseline ResFlow to account for the additional parameters and depth introduced by our method.
Note that these models were significantly smaller than those used by \citet{chen2019residual}: e.g.\ for CIFAR10, the ResFlow used by \citet{chen2019residual} had 25M parameters, while our two baseline ResFlows and our \modelacr-ResFlow had 2.4M, 6.2M, and 5.6M parameters respectively.
We likewise considered RealNVPs with the same multi-scale convolutional architecture used by  \citet{dinh2016density} for their CIFAR-10 experiments.
For these runs we trained baseline RealNVPs, corresponding \modelacr-RealNVPs, and larger baseline RealNVPs with more depth and parameters.

The results are given in \autoref{tab:images-resflow} and \autoref{tab:images}.
Observe \modelacr s outperformed the baseline models for all datasets, which shows that our approach can scale to high dimensions.
For the \modelacr-ResFlows, we also obtained better performance than \citet{chen2019residual} on MNIST and better performance than Glow \citep{kingma2018glow} on CIFAR10, despite using a much smaller model.
Samples from all models are shown in \autoref{sec:images} of the Supplement.

\begin{table}%
\caption{Average test bits per dimension.\protect\footnotemark\ Lower is better.}
\label{tab:images-resflow}
\begin{center}
\begin{small}
\begin{sc}
\begin{tabular}{lll}
\toprule
                          &\multicolumn{1}{c}{MNIST} &\multicolumn{1}{c}{CIFAR-10} \\
\midrule
ResFlow (small)           & $1.074$                 & $3.474$ \\
ResFlow (big)             & $1.018$                 & $3.422$ \\
\modelacr-ResFlow         & ${\bf 0.922}$           & ${\bf 3.334}$ \\
\bottomrule
\end{tabular}
\end{sc}
\end{small}
\end{center}
\end{table}

\begin{table}%
\caption{Mean $\pm$ standard error of average test set bits per dimension over 3 random seeds. Lower is better.} %
\label{tab:images}
\begin{center}
\begin{small}
\begin{sc}
\begin{tabular}{lll}
\toprule
                          &\multicolumn{1}{c}{Fashion-MNIST} &\multicolumn{1}{c}{CIFAR-10} \\
\midrule
RealNVP (small)             & $2.944 \pm 0.003$           & $3.565 \pm 0.001$ \\
RealNVP (big)            & $2.946 \pm 0.002$           & $3.554 \pm 0.001$ \\
\modelacr-RealNVP        & ${\bf 2.823 \pm 0.003}$     & ${\bf 3.477 \pm 0.019}$ \\
\bottomrule
\end{tabular}
\end{sc}
\end{small}
\end{center}
\vspace{-1em}
\end{table}

\footnotetext{Only one seed was used per run due to computational limitations. However, the results were not cherry-picked.}

\section{Conclusion and Future Work}

The constraint \eqref{eq:homeomorphic-necessary-condition} shows that normalising flows are unable to exactly model targets whose topology differs from that of the prior.
Moreover, in order to approximate such targets closely, \autoref{thm:exploding-bilipschitz-constant} shows that the bi-Lipschitz constant of a flow must become arbitrarily large.
To address these problems, we have proposed \modelacr s, which can ``clean up'' regions of mass that are placed outside the support of the target by a standard flow.
\modelacr s perform well in practice and outperform baseline flows on several benchmark datasets.

While we have focussed on the use of \modelacr s for density estimation in this paper, it would also be interesting to apply \modelacr s in other contexts where normalising flows have been used successfully.
As \modelacr s do not have an analytically available density, this would likely require the modification of existing numerical frameworks, but the expressiveness benefits provided by \modelacr s might make this additional effort worthwhile.
We leave this direction for future work.

\section*{Acknowledgements}

Rob Cornish is supported by the EPSRC Centre for Doctoral Training in Autonomous Intelligent Machines \& Systems (EP/L015897/1) and NVIDIA. 
Anthony Caterini is a Commonwealth Scholar supported by the U.K.\ Government.
Arnaud Doucet is partially supported by the U.S.\ Army Research
Laboratory, the U.S.\ Army Research Office, and by the U.K.\ Ministry of
Defence (MoD) grant EP/R013616/1 and the U.K.\ EPSRC under grant numbers
EP/R034710/1 and EP/R018561/1.

\nocite{dupont2019augmented}
\nocite{ho2019flow++}

\bibliography{main}
\bibliographystyle{icml2020}

\clearpage
\onecolumn

\icmltitle{\mytitle: Supplementary Material}

\appendix
\numberwithin{equation}{section}
\renewcommand\thefigure{\thesection.\arabic{figure}}
\renewcommand\thetable{\thesection.\arabic{table}}

\section{Guide to Notation} \label{sec:notation}

\begin{table}[h]
  \begin{tabularx}{\linewidth}{l X}
    $(a_n)$ & A sequence of elements $a_1, a_2, \ldots$ \\
    $a(n) = \Theta(b(n))$ & $a(n)$ differs from $b(n)$ by at most a constant factor as $n \to \infty$ \\
    $u \odot v$ & The elementwise product of tensors $u$ and $v$ \\
    $\LogSumExp(a_1, \ldots, a_m)$ & $\log\left(\sum_{i=1}^m \exp(a_i)\right)$ \\
    $e^v$, where $v \in \R^d$ & $(e^{v_1}, \ldots, e^{v_d})$ \\
    $\norm{v}$ & The norm of a vector $v \in \R^d$ (our results are agnostic to the specific choice of $\norm{\cdot}$) \\
    $\opnorm{A}$ & The operator norm of a matrix $A \in \R^{d_1 \times d_2}$ induced by $\norm{\cdot}$ \\
    $\idmat_d$ & The $d \times d$ identity matrix \\
    $\det A$ & The determinant of a square matrix $A$ \\
    $\jac f(z)$ & The Jacobian matrix of a function $f$ evaluated at $z$ \\
    $\jac F(z; u)$ & The Jacobian matrix of a function $\jac F(\cdot; u)$ (i.e. with $u$ fixed) evaluated at $z$ \\
    $\Lip f$ & The Lipschitz constant of a function $f$ \\
    $\BiLip f$ & The bi-Lipschitz constant of a function $f$ \\
    $\mathcal{A} \cong \mathcal{B}$ & The topological spaces $\mathcal{A}$ and $\mathcal{B}$ are homeomorphic \\
    $\closure{B}$ & The topological closure of a set $B$ \\
    $\interior(B)$ & The interior of a set $B$ \\
    $\boundary B$ & The boundary of a set $B$ \\
    $\supp \mu$ & The support of a measure $\mu$ \\
    $f \# \mu$ & The pushforward of a measure $\mu$ by a function $f$ \\
    $\mu_n \Dto \mu$ & Weak convergence of the measures $\mu_n$ to $\mu$ \\
  \end{tabularx}
\end{table}

\section{Proofs}

\subsection{Preliminaries} \label{sec:topological-preliminaries}

We require some basic results that we include here for completeness.
We will make use of standard definitions and results from topology and real analysis.
A complete background to these topics can be found in \citet{dudley2002real}.

\subsubsection{Supports of Measures}

Recall that for a Borel measure $\mu$ on a topological space $\mathcal{Z}$, the \emph{support} of $\mu$, denoted $\supp \mu$, is the set of all $z \in \mathcal{Z}$ such that $\mu(N_z) > 0$ for every open set $N_z$ containing $z$.

The following is an immediate consequence:

\begin{proposition} \label{prop:support-absolute-continuity}
    Suppose $\mu$ and $\nu$ are Borel measures with $\mu$ absolutely continuous with respect to $\nu$. Then
    \[
        \supp \mu \subseteq \supp \nu.
    \]
\end{proposition}

\begin{proof}
    Suppose $z \not \in \supp \nu$. Then there exists an open set $N_z$ containing $z$ such that $\nu(N_z) = 0$. By absolute continuity, we have also that $\mu(N_z) = 0$ and hence $z \not \in \supp \mu$.
\end{proof}

In general the converse need not hold. For example, the Dirac measure on $0$ has support contained within the Lebesgue measure on $\R$ (which has full support), but is not absolutely continuous with respect to it.

The following characterisation is useful:

\begin{proposition} \label{lem:support-closed}
    For any Borel measure $\mu$,
    \begin{equation} \label{eq:support-complement-as-union}
        (\supp \mu)^c = \bigcup_{\substack{\text{$A$ $\mathrm{open}$:} \\ \mu(A) = 0}} A,
    \end{equation}
    and hence $\supp \mu$ is closed.
\end{proposition}

\begin{proof}
    This follows directly from the definitions, since $z \not \in \supp \mu$ if and only if there exists open $N_z$ with $z \in N_z$ and $\mu(N_z) = 0$, which is just another way of saying that $z$ is contained in the right-hand side of \eqref{eq:support-complement-as-union}.
    It follows that $(\supp \mu)^c$ is open, and hence $\supp \mu$ is closed.
\end{proof}

We mainly care about how the support of a measure is transformed by a pushforward function. The following proposition characterises what occurs in this case.

\begin{proposition} \label{lem:pushforward-support}
    Suppose $\mathcal{Z}$ and $\mathcal{X}$ are topological spaces. If $\mu$ is a Borel measure on $\mathcal{Z}$ such that $\mu((\supp \mu)^c) = 0$, and if $f: \mathcal{Z} \to \mathcal{X}$ is continuous, then
    \[
        \supp f\#\mu = \closure{f(\supp \mu)}.
    \]
\end{proposition}

\begin{proof}
    Suppose $x \not \in \closure{f(\supp\mu)}$. Then $x$ must have an open neighbourhood $N_x$ such that
    \[
        N_x \cap f(\supp \mu) = \emptyset.
    \]
    This implies
    \begin{align*}
        f^{-1}(N_x) \cap \supp \mu
        &\subseteq f^{-1}(N_x) \cap f^{-1}(f(\supp \mu)) \\
        &= f^{-1}(N_x \cap f(\supp \mu)) \\
        &= f^{-1}(\emptyset) \\
        &= \emptyset.
    \end{align*}
    We then have
    \[
        f\#\mu(N_x) = \mu(f^{-1}(N_x)) = \mu(f^{-1}(N_x) \cap \supp \mu) = 0,
    \]
    where the second equality follows since we assumed $\mu((\supp \mu)^c) = 0$,
    and hence $x \not \in \supp f\#\mu$. 
    Consequently
    \[
        \supp f\#\mu \subseteq \closure{f(\supp \mu)}.
    \]
    
    In the other direction, suppose $x \in f(\supp \mu)$, so that $x = f(z)$ for some $z \in \supp \mu$. Given an open neighbourhood $N_x$ it then follows from continuity that $f^{-1}(N_x)$ is an open neighbourhood of $z$, and so
    \[
        f\#\mu(N_x) = \mu(f^{-1}(N_x)) > 0
    \]
    since $z \in \supp \mu$. This entails $\supp f\#\mu \supseteq f(\supp \mu)$, which means
    \[
        \supp f\#\mu = \closure{\supp f\#\mu} \supseteq \closure{f(\supp \mu)}
    \]
    by \autoref{lem:support-closed}.
\end{proof}

Note that in general we need not have $\supp f\#\mu = f(\supp \mu)$. For example, if $\mu$ is Gaussian and $f = \arctan$, then
\[
    f(\supp \mu) = (-1, 1) \neq [-1, 1] = \supp f\#\mu.
\]
Likewise, in general we do require the assumption $\mu((\supp \mu)^c) = 0$. This is because there exist examples of nontrivial Borel measures $\mu$ such that $\supp \mu = \emptyset$. 
Taking $f \equiv x_0$ to be any constant $x_0 \in \mathcal{X}$ (in which case $f$ is certainly continuous) then gives
\[
    \closure{f(\supp \mu)} = \emptyset \neq \{x_0\} = \supp f\#\mu.
\]
However, for our purposes, the following proposition shows that this is not a restriction.

\begin{proposition} \label{lem:support-all-that-matters}
    Suppose $\mu$ is a Borel measure on a separable metric space $\mathcal{Z}$. Then
    \[
        \mu((\supp \mu)^c) = 0.
    \]
\end{proposition}

\begin{proof}
    Throughout the proof, for each $z$ and $r > 0$, we will denote by $B(z, r)$ an open ball of radius $r$ centered at $z$.
    Likewise, for each $z \not \in \supp \mu$, let
    \[
        r^\star(z) \coloneqq \sup \{r > 0 \mid \mu(B(z, r)) = 0\}.
    \]
    Observe that $r^\star$ is well-defined (but possibly infinite) since $z \not \in \supp \mu$ means there must exist some $r > 0$ such that $\mu(B(z, r)) = 0$.

    We first show that $\mu(B(z, r^\star(z))) = 0$ for all $z \not \in \supp \mu$.
    To this end, fix $z$ and choose a sequence $r_m \uparrow r^\star(z)$ with $r_m < r^\star(z)$.
    We then have
    \[
        B(z, r^\star(z)) = \bigcup_{m=1}^\infty B(z, r_m),
    \]
    and so
    \[
        \mu(B(z, r^\star(z))) = \lim_{m\to \infty} \mu(B(z, r_m)) = 0
    \]
    by continuity of measure.

    Now, by separability, we can choose a countable sequence $(z_k) \subseteq (\supp \mu)^c$ such that $\overline{\{z_k\}} = \overline{(\supp \mu)^c}$. 
    We show that
    \[
        (\supp \mu)^c = \bigcup_{k=1}^\infty B(z_k, r^\star(z_k)),
    \]
    from which the result follows by countable subadditivity.
    It is clear from \eqref{eq:support-complement-as-union} that the left-hand side is a superset of the right.
    In the other direction, let $z \in (\supp \mu)^c$.
    By construction of $(z_k)$, there exists a subsequence $(z_{k'})$ such that $z_{k'} \to z$.
    For all $k'$ large enough we then have $z_{k'} \in B(z, r^\star(z)/2)$ and hence
    \[
        B(z_{k'}, r^\star(z)/2) \subseteq B(z, r^\star(z))
    \]
    by triangle inequality.
    It follows that for such $k'$ we have
    \[
        \mu(B(z_{k'}, r^\star(z)/2)) \leq \mu(B(z, r^\star(z))) = 0,
    \]
    and so $r^\star(z_{k'}) \geq r^\star(z)/2$ since $r^\star(z_{k'})$ is the supremum.
    But then we have
    \[
        z \in B(z_{k'}, r^\star(z)/2) \subseteq B(z_{k'}, r^\star(z_{k'})),
    \]
    so that
    \[
        z \in \bigcup_{k=1}^\infty B(z_k, r^\star(z_k))\\
    \]
    and we are done.
\end{proof}

\subsection{Lipschitz and Bi-Lipschitz Functions} \label{sec:lip-and-bilip}

We assume that $\noisespace \subseteq \R^{d_\noisespace}$, $\dataspace \subseteq \R^{d_\dataspace}$, and $\fwdmap : \noisespace \to \dataspace$.
Recall that the \emph{Lipschitz} constant of $\fwdmap$, denoted $\Lip \fwdmap$, is defined as the infimum over $\lipconst \in [0, \infty]$ such that
\[
    \norm{\fwdmap(\noise) - \fwdmap(\noise')} \leq \lipconst \norm{\noise - \noise'}
\]
for all $\noise, \noise' \in \noisespace$.
Likewise the \emph{bi-Lipschitz} constant $\BiLip \fwdmap$ is defined as the infimum over $\lipconst \in [1, \infty]$ such that
\[
     \lipconst^{-1} \norm{\noise - \noise'} \leq \norm{\fwdmap(\noise) - \fwdmap(\noise')} \leq \lipconst \norm{\noise - \noise'}
\]
for all $\noise, \noise' \in \noisespace$.
We prove some basic properties that follow from this definition.

\begin{proposition} \label{prop:bilip-const-finite-cond}
    $\BiLip \fwdmap < \infty$ if and only if $\fwdmap$ is injective and $\max(\Lip \fwdmap, \Lip \fwdmap^{-1}) < \infty$, where $\fwdmap^{-1} : \fwdmap(\noisespace) \to \noisespace$.
    For all injective $\fwdmap$, we then have $\BiLip \fwdmap = \max(\Lip \fwdmap, \Lip \fwdmap^{-1})$.
\end{proposition}

\begin{proof}
    For the first statement, suppose $\BiLip \fwdmap < \infty$.
    It is immediate that $\BiLip \fwdmap \geq \Lip \fwdmap$.
    To see that $\fwdmap$ is injective, note that for $\noise \neq \noise'$ we have
    \[
        \norm{\fwdmap(\noise) - \fwdmap(\noise')} \geq
        (\BiLip \fwdmap)^{-1} \norm{\noise - \noise'} > 0
    \]
    and so $\fwdmap(\noise) \neq \fwdmap(\noise')$.
    On the other hand, for $\data, \data' \in \fwdmap(\noisespace)$, we have
    \[
        (\BiLip\fwdmap)^{-1} \norm{\fwdmap^{-1}(\data) - \fwdmap^{-1}(\data')} \leq \norm{\fwdmap(\fwdmap^{-1}(\data)) - \fwdmap(\fwdmap^{-1}(\data'))} = \norm{\data - \data'},
    \]
    which gives that $\BiLip \fwdmap \geq \Lip \fwdmap^{-1}$.
    Altogether we have
    \begin{equation} \label{eq:max_leq_bilip}
         \max(\Lip \fwdmap, \Lip \fwdmap^{-1}) \leq \BiLip \fwdmap < \infty,
    \end{equation}
    which gives the forward direction.
    
    Next suppose $\fwdmap$ is injective and that
    \[
        \lipconst \coloneqq \max(\Lip \fwdmap, \Lip \fwdmap^{-1}) < \infty.
    \]
    For $\noise, \noise' \in \noisespace$, we certainly have
    \[
        \norm{\fwdmap(\noise) - \fwdmap(\noise')} \leq \lipconst \norm{\noise - \noise'}.
    \]
    Likewise, since $\fwdmap(\noise), \fwdmap(\noise') \in \fwdmap(\noisespace)$,
    \[
        \norm{\noise - \noise'} = \norm{\fwdmap^{-1}(\fwdmap(\noise)) - \fwdmap^{-1}(\fwdmap(\noise'))} \leq \lipconst \norm{\fwdmap(\noise) - \fwdmap(\noise')},
    \]
    so that
    \[
        \lipconst^{-1} \norm{\noise - \noise'} \leq \norm{\fwdmap(\noise) - \fwdmap(\noise')}
    \]
    because injectivity of $\fwdmap$ means that $\lipconst > 0$.
    From this it follows that
    \begin{equation} \label{eq:bilip_leq_max}
        \BiLip \fwdmap \leq \lipconst < \infty,
    \end{equation}
    which gives the reverse direction, proving the first statement.
    
    For the second statement, suppose $\fwdmap$ is injective.
    Then if $\BiLip \fwdmap < \infty$, \eqref{eq:max_leq_bilip} and \eqref{eq:bilip_leq_max} together give
    \[
        \BiLip \fwdmap = \max(\Lip \fwdmap, \Lip \fwdmap^{-1}).
    \]
    On the other hand, if $\BiLip \fwdmap = \infty$ then $\max(\Lip \fwdmap, \Lip \fwdmap^{-1}) = \infty$ since we would otherwise obtain a contradiction by the first statement of the proposition.
    This completes the proof.
\end{proof}

It follows directly that if $\BiLip \fwdmap < \infty$, then $\fwdmap$ is a homeomorphism from $\noisespace$ to $\fwdmap(\noisespace)$.\footnote{Note however that the converse is not true in general: for example, $\exp$ is a homeomorphism from $\R$ to $(0, \infty)$, but $\BiLip \exp = \infty$.} Moreover, in this case $\fwdmap$ maps closed sets to closed sets, as the following result shows:

\begin{proposition} \label{prop:bilip-image-closed}
    If $\BiLip \fwdmap < \infty$ and $\noisespace$ is closed in $\R^{d_\noisespace}$, then $\fwdmap(\noisespace)$ is closed in $\R^{d_\dataspace}$.
\end{proposition}

\begin{proof}
    It is a straightforward consequence of \autoref{prop:bilip-const-finite-cond} that if $(\data_n) \subseteq \fwdmap(\noisespace)$ is Cauchy, then $(\fwdmap^{-1}(\data_n))$ is Cauchy.
    Consequently $(\fwdmap^{-1}(\data_n))$ converges to some $\noise_\infty \in \noisespace$, since $\noisespace$ is a closed subset of a complete space and therefore complete.
    But then
    \begin{align*}
        \norm{\data_n - \fwdmap(\noise_\infty)} &= \norm{\fwdmap(\fwdmap^{-1}(\data_n)) - \fwdmap(\noise_\infty)} \\
        &\leq \lipconst \norm{\fwdmap^{-1}(\data_n) - \noise_\infty} \\
        &\to 0
    \end{align*}
    as $n \to \infty$.
    Consequently $\fwdmap(\noisespace)$ is complete, and so $\fwdmap(\noisespace)$ is closed as desired since the ambient space $\R^{d_\dataspace}$ is complete.
\end{proof}

The Lipschitz constant can be computed from the \emph{operator norm} $\opnorm{\cdot}$ of the Jacobian of $\fwdmap$.
Recall that $\opnorm{\cdot}$ is defined as for a matrix $A \in \R^{d_\dataspace \times d_\noisespace}$ as
\[
    \opnorm{A} \coloneqq \sup_{\substack{v \in \R^{d_\noisespace} : \\ \norm{v} = 1}} \norm{Av}
\]
where we think of elements of $\R^{d_\noisespace}$ as column vectors.

\begin{proposition} \label{prop:lip-const-op-norm-expression}
    If $\noisespace = \R^{d_\noisespace}$, $\dataspace = \R^{d_\dataspace}$, and $\fwdmap$ is everywhere differentiable, then
    \[
        \Lip \fwdmap = \sup_{\noise \in \noisespace} \opnorm{\jac \fwdmap(\noise)}.
    \]
\end{proposition}

\begin{proof}
    If $v \in \noisespace$ with $\norm{v} = 1$, then
    \begin{align*}
        \norm{[\jac \fwdmap(\noise)] v} &= \lim_{t \to 0} \frac{\norm{\fwdmap(\noise + tv) - \fwdmap(\noise)}}{\abs{t}} \\
        &\leq \lim_{t \to 0} \frac{(\Lip \fwdmap) \norm{(\noise + tv) - \noise}}{\abs{t}} \\
        &= \Lip \fwdmap.
    \end{align*}
    It follows directly that
    \[
        \opnorm{\jac \fwdmap(\noise)} \leq \Lip \fwdmap.
    \]
    On the other hand, suppose $\Lip \fwdmap > \lipconst$.
    Then there exists $\noise, \noise' \in \noisespace$ such that
    \[
        \norm{\fwdmap(\noise) - \fwdmap(\noise')} > \lipconst \norm{\noise - \noise'}.
    \]
    Since $\fwdmap$ is differentiable, so too is the map $\varphi : [0, 1] \to \dataspace$ defined by
    \[
        \varphi(t) \coloneqq \fwdmap(t \noise + (1 - t) \noise').
    \]
    By Theorem 5.19 of \citet{rudin1964principles}, there exists $t_0 \in (0, 1)$ such that the derivative $\varphi'$ satisfies
    \[
        \norm{\varphi'(t_0)} \geq \norm{\fwdmap(\noise') - \fwdmap(\noise)} > \lipconst \norm{\noise - \noise'}.
    \]
    But, letting $\noise_0 \coloneqq t_0 \noise + (1 - t_0) \noise'$, observe that
    \begin{align*}
        \varphi'(t_0) &= \lim_{t \to 0} \frac{\fwdmap(\noise_0 + t(\noise - \noise')) - \fwdmap(\noise_0)}{t} \\
        &= [\jac \fwdmap(\noise_0)](\noise - \noise'),
    \end{align*}
    where we think of $z, z'$ as column vectors.
    As such,
    \begin{align*}
        \opnorm{\jac \fwdmap(\noise_0)} \norm{\noise - \noise'} &\geq \norm{[\jac \fwdmap(\noise_0)] (\noise - \noise')} \\
        &= \norm{\varphi'(t_0)} \\
        &> \lipconst \norm{\noise - \noise'}
    \end{align*}
    and so
    \[
        \sup_{\noise \in \noisespace} \opnorm{\jac \fwdmap(\noise)} > \lipconst.
    \]
    Since $\lipconst$ was arbitrary this means that
    \[
        \Lip \fwdmap \leq \sup_{\noise \in \noisespace} \opnorm{\jac \fwdmap(\noise)}
    \]
    which gives the result.
\end{proof}

\autoref{prop:bilip-const-finite-cond} and \autoref{prop:lip-const-op-norm-expression} then immediately entail the following:

\begin{corollary}
    Suppose $\noisespace = \R^{d_\noisespace}$ and $\dataspace = \R^{d_\dataspace}$. If $\fwdmap$ is injective, and if $\fwdmap$ and $\fwdmap^{-1} : \fwdmap(\noisespace) \to \noisespace$ are everywhere differentiable, then
    \[
        \BiLip \fwdmap = \max\left(\sup_{\noise \in \noisespace} \opnorm{\jac \fwdmap(\noise)}, \sup_{\data \in \fwdmap(\noisespace)} \opnorm{\jac \fwdmap^{-1}(\data)} \right).
    \]
\end{corollary}

\subsubsection{Arzel\`a-Ascoli}

Our proof of \autoref{thm:exploding-bilipschitz-constant} makes use of the Arzel\`a-Ascoli theorem. This is a standard and foundational result in analysis, but we include a statement here for completeness. To this end, suppose we have a sequence of functions $f_n : \mathcal{Z} \subseteq \R^{d_\mathcal{X}} \to \mathcal{X} \subseteq \R^{d_{\mathcal{X}}}$. We say that $(f_n)$ is \emph{pointwise bounded} if, for all $z \in \mathcal{Z}$,
\[
    \sup_{n} \norm{f_n(z)} < \infty.
\]
Likewise, $(f_n)$ is \emph{uniformly equicontinuous} if for every $\epsilon > 0$ there exists $\delta > 0$ such that, for all $n$,
\[
    \norm{f_n(z) - f_n(z')} < \epsilon
\]
whenever $\norm{z - z'} < \delta$.

\begin{theorem}[Arzel\`a-Ascoli] \label{thm:arzela-ascoli}
    If a sequence of functions $f_n : \mathcal{Z} \subseteq \R^{d_\mathcal{Z}} \to \mathcal{X} \subseteq \R^{d_\mathcal{X}}$ is pointwise bounded and uniformly equicontinuous, then there exists a subsequence of $(f_n)$ that converges uniformly on every compact subset of $\mathcal{Z}$.
\end{theorem}

\begin{proof}
    The case $d = 1$ is proven for example by \citet[Theorem 11.28]{rudin2006real}. This can be extended to the case $d > 1$ by a standard argument. In particular, write
    \[
        f_n =: (f_{n,1}, \ldots, f_{n,d}),
    \]
    where $f_{n,i} : \mathcal{Z} \to \R$. Then extract a subsequence $(f_{n_1})$ of $(f_n)$ such that $f_{n,1}$ converges uniformly on every compact subset of $\mathcal{Z}$. Then extract a subsequence of $(f_{n_1})$ such that the same holds for $f_{n,2}$, and so on. The result is a subsequence $(f_{n'})$ such that each $f_{n',i}$ converges uniformly on compact subsets of $\mathcal{Z}$, from which the same holds for $f_{n'}$ also by the triangle inequality.
\end{proof}

\subsection{Pushforward Maps Require Unbounded Bi-Lipschitz Constants} \label{sec:exploding-bilipschitz-constant-proof}

\lipthm*

\begin{proof}
    We suppose that $\pushforward{\fwdmap_n}{\priormeas} \Dto \targetmeas$ and prove the contrapositive. That is, without loss of generality (pass to a subsequence if necessary) we assume
    \begin{equation} \label{eq:bilip-assumption}
        \lipconst \coloneqq \sup_n \BiLip f_n < \infty,
    \end{equation}
    and prove that $\supp \priormeas \cong \supp \targetmeas$.
    
    We first show that $(\fwdmap_n)$ is pointwise bounded.
    To this end, observe that Prokhorov's theorem \citep[Proposition 9.3.4]{dudley2002real} means that $\priormeas$ is tight and that the sequence $(\pushforward{\fwdmap_n}{\priormeas})$ is uniformly tight.
    As such, there exists compact $K \subseteq \R^{d_\noisespace}$ such that $\priormeas(K) > 0$, and compact $K' \subseteq \R^{d_\dataspace}$ such that
    \[
        \inf_n \pushforward{\fwdmap_n}{\priormeas}(K') > 1 - \priormeas(K).
    \]
    For each $n$, we must then have some $\noise_n \in K$ such that $\fwdmap_n(\noise_n) \in K'$; otherwise $K' \subseteq \fwdmap_n(K)^c$ and so
    \begin{align*}
        \pushforward{\fwdmap_n}{\priormeas}(K')
        &\leq \pushforward{\fwdmap_n}{\priormeas}(\fwdmap_n(K)^c) \\
        &= 1 - \pushforward{\fwdmap_n}{\priormeas}(\fwdmap_n(K)) \\
        &= 1 - \priormeas(\fwdmap_n^{-1}(\fwdmap_n(K))) \\
        &= 1 - \priormeas(K)
    \end{align*}
    since $\fwdmap_n$ is injective by \autoref{prop:bilip-const-finite-cond}.
    But for any fixed $\noise \in \R^{d_\noisespace}$, this entails
    \begin{align*}
        \sup_{n} \norm{\fwdmap_n(\noise)} &\leq \sup_n \norm{\fwdmap_n(\noise_n)} +  \norm{\fwdmap_n(\noise) - \fwdmap_n(\noise_n)}  \\
        &\leq \sup_{\data \in K'} \norm{x} + \sup_{\noise \in K} \lipconst \norm{\noise - \noise_n} \\
        &\leq \sup_{\data \in K'} \norm{x} + 2 \lipconst \sup_{\noise \in K} \norm{\noise} \\
        &< \infty
    \end{align*}
    since $K$ and $K'$ are compact.

    Next, observe that \eqref{eq:bilip-assumption} easily means $(\fwdmap_n)$ is uniformly equicontinuous. In particular, for $\epsilon > 0$, choosing $\delta \coloneqq \epsilon / \lipconst$ gives
    \[
        \norm{\fwdmap_n(\noise) - \fwdmap_n(\noise')} \leq \lipconst \norm{\noise - \noise'} < \epsilon
    \]
    for all $n$ whenever $\norm{\noise - \noise'} < \delta$
    
    \autoref{thm:arzela-ascoli} now entails the existence of a subsequence $(\fwdmap_{n'})$ that converges uniformly on every compact subset of $\R^{d_\noisespace}$.
    In particular, $(\fwdmap_{n'})$ converges pointwise to a limit that we denote by $\fwdmap_\infty$.
    Moreover, $\fwdmap_\infty$ is bi-Lipschitz.
    To see this, recall that for all $n'$ and $\noise, \noise' \in \R^{d_\noisespace}$ we have
    \[
        \frac{1}{\lipconst} \norm{\noise - \noise'} \leq \norm{\fwdmap_{n'}(\noise) - \fwdmap_{n'}(\noise')} \leq \lipconst \norm{\noise - \noise'},
    \]
    by our assumption \eqref{eq:bilip-assumption}.
    Taking $n' \to \infty$ shows that $\BiLip \fwdmap_\infty \leq \lipconst < \infty$.

    We also have that
    \begin{equation} \label{eq:pushforward-subsequence-limit}
        \pushforward{\fwdmap_{n'}}{\priormeas} \Dto \pushforward{\fwdmap_\infty}{\priormeas}.
    \end{equation}
    This follows from the Portmanteau theorem \citep[Theorem 11.3.3]{dudley2002real}.
    In particular, suppose $h$ is a bounded Lipschitz function, and let $B_r \subseteq \R^{d_\noisespace}$ denote a ball of radius $r > 0$ at the origin. Then
    \begin{align*}
        \Abs{\int h(\data) \,\pushforward{\fwdmap_{n'}}{\priormeas}(\diff \data) - \int h(\data) \, \pushforward{\fwdmap_\infty}{\priormeas}(\diff\data)}
            &= \Abs{\int h(\fwdmap_{n'}(\noise)) - h(\fwdmap_\infty(\noise)) \, \priormeas(\diff \noise)} \\
            &\leq \int_{B_r} \Abs{h(\fwdmap_{n'}(\noise)) - h(\fwdmap_\infty(\noise))} \priormeas(\diff \noise) \\
                &\qquad + \int_{B_r^c} \Abs{h(\fwdmap_{n'}(\noise))} + \Abs{h(\fwdmap_\infty(\noise))} \priormeas(\diff \noise) \\
            &\leq \priormeas(B_r) (\Lip h) \sup_{\noise \in B_r} \norm{\fwdmap_n(\noise) - \fwdmap_\infty(\noise)} + 2 \priormeas(B_r^c) \sup_{\noise \in \R^{d_\noisespace}} \Abs{h(z)}.
    \end{align*}
    Hence
    \[
        \limsup_{n' \to \infty} \Abs{\int h(\data) \, \pushforward{\fwdmap_{n'}}{\priormeas}(\diff \data) - \int h(\data) \pushforward{\fwdmap_\infty}{\priormeas}(\diff \data)} \leq 2 \priormeas(B_r^c) \sup_{\noise \in \R^{d_\noisespace}} \Abs{h(\noise)}
    \]
    by the uniform convergence of $\fwdmap_{n'}$ to $\fwdmap_\infty$ on compact subsets, and since $\Lip h < \infty$. Taking $r \to \infty$, the right-hand side vanishes since $h$ is bounded, and we obtain \eqref{eq:pushforward-subsequence-limit}.
    
    We are now ready to complete the proof.
    Since $\fwdmap_\infty$ is bi-Lipschitz, \autoref{prop:bilip-const-finite-cond} means that $\fwdmap_\infty$ is a homeomorphism from $\R^{d_\noisespace}$ to $\fwdmap_\infty(\R^{d_\noisespace})$.
    This certainly gives
    \[
        \supp \priormeas \cong \fwdmap_\infty(\supp \priormeas).
    \]
    But now \autoref{prop:bilip-image-closed} means %
    \[
        \fwdmap_\infty(\supp \priormeas) = \closure{\fwdmap_\infty(\supp \priormeas)}
    \]
    where the closure is taken in $\R^{d_\dataspace}$.
    However, from \eqref{eq:pushforward-subsequence-limit} we have
    \[
        \targetmeas = \pushforward{\fwdmap_\infty}{\priormeas},
    \]
    which by \autoref{lem:pushforward-support} means that
    \[
        \supp \targetmeas = \supp \pushforward{\fwdmap_\infty}{\priormeas} = \closure{\fwdmap_\infty(\supp \priormeas)}.
    \]
    Consequently
    \[
        \supp \targetmeas = \fwdmap_\infty(\supp \priormeas) \cong \supp \priormeas
    \]
    as desired.
\end{proof}

The following corollary extends the above result to the case where $\supp \targetmeas$ may be homeomorphic to $\supp \priormeas$, but $\targetmeas$ is very \emph{close} to a probability measure with non-homeomorphic support to $\priormeas$.
Here $\rho$ denotes any metric for the weak topology.
In other words, $\rho$ must be a metric on the space of distributions that satisfies $\rho(P_n, P) \to 0$ as $n \to \infty$ if and only if $P_n \Dto P$.
The L\'evy-Prokhorov and bounded Lipschitz metrics provide standard examples of such $\rho$ \citep[Definition 3.3.10]{villani2008optimal}.

\lipthmcorr*

\begin{proof}
    Define $\lipconst : [0, \infty) \to [1, \infty]$ by
    \[
        \lipconst(\epsilon) \coloneqq \inf\left\{\BiLip \fwdmap \mid \fwdmap : \R^{d_\noisespace} \to \R^{d_\dataspace}, \rho(\pushforward{\fwdmap}{\priormeas}, P_\datavar^0) \leq 2 \epsilon \right\},
    \]
    with $\lipconst(\epsilon) \coloneqq \infty$ if the infimum is taken over the empty set.
    Certainly $\lipconst$ is nonincreasing.
    If we have both $\rho(\targetmeas, P_\datavar^0) \leq \epsilon$ and $\rho(\pushforward{\fwdmap}{\priormeas}, \targetmeas) \leq \epsilon$, then the triangle inequality gives
    \[
        \rho(\pushforward{\fwdmap}{\priormeas}, P_\datavar^0) \leq \rho(\pushforward{\fwdmap}{\priormeas}, \targetmeas) + \rho(\targetmeas, P_\datavar^0) \leq 2 \epsilon
    \]
    and so $\BiLip \fwdmap \geq \lipconst(\epsilon)$ since the right-hand side is an infimum.
    It remains only to show that $\lipconst(\epsilon) \to \infty$ as $\epsilon \to 0$.
    For contradiction, suppose there exists $\epsilon_n \to 0$ such that $\sup_n \lipconst(\epsilon_n) < \infty$.
    From the definition of $\lipconst$, this means that for each $n$ there exists $\fwdmap_n : \R^{d_\noisespace} \to \R^{d_\dataspace}$ such that $\rho(\pushforward{\fwdmap_n}{\priormeas}, P_\datavar^0) \leq 2 \epsilon_n$ and $\BiLip \fwdmap_n \leq \lipconst(\epsilon_n) + 1$.
    It follows directly that $\rho(\pushforward{\fwdmap_n}{\priormeas}, P_\datavar^0) \to 0$ as $n \to \infty$, which in turn means $\pushforward{\fwdmap_n}{\priormeas} \Dto P_\datavar^0$ since $\rho$ is a metric for the weak topology.
    At the same time we have
    \[
        \sup_n \BiLip \fwdmap_n \leq \sup_n \lipconst(\epsilon_n) + 1 < \infty,
    \]
    which contradicts \autoref{thm:exploding-bilipschitz-constant}, since we assumed $\supp \priormeas \not \cong \supp P_\datavar^0$.
\end{proof}

\subsection{Variance of the Russian Roulette Estimator} \label{sec:russian-roulette-variance}

In this section we briefly review the Russian roulette estimator used in \citet{chen2019residual}, and then discuss some scenarios in which we expect the variance of this estimator to increase unboundedly.

\subsubsection{Russian Roulette Estimator}

Residual Flows (ResFlows, \cite{chen2019residual}), building off of Invertible Residual Networks (iResNets, \cite{behrmann2019invertible}), model the data by repeatedly stacking bijections of the form $\fwdmap_\layeridx^{-1}(\data) = \data + g_\layeridx(\data),$ where $\Lip g_\layeridx =: \kappa < 1,$ as mentioned in \eqref{eq:resflow}.
The change-of-variable formula for one layer of flow reads as, for $x \in \R^d$,
\begin{equation} \label{eq:cov_ps}
    \log p_X(x) = \log p_Z(\fwdmap_\layeridx^{-1}(x)) + \text{tr} \left( \sum_{j=1}^\infty \frac{(-1)^{j+1}}{j} \jac g_\layeridx(x)^j \right).
\end{equation}
To deal with this infinite series, iResNets truncate after a fixed number of terms -- this provides a biased estimate of the log-likelihood of a point $x$ under the model.
ResFlows rely on an alternative method of estimating \eqref{eq:cov_ps}, first using a Russian roulette procedure to rewrite the series as follows:
\[
    \sum_{j=1}^\infty \frac{(-1)^{j+1}}{j} \text{tr} \left( \jac g_\layeridx(x)^j \right) = \E_N \left[ \sum_{j=1}^N \frac{(-1)^{j+1}}{j} \frac{\text{tr}\left( \jac g_\layeridx(x)^j\right)}{p_j}  \right] \eqqcolon S(x), 
\]
where $N \sim \text{Geom}(p)$ is a geometric random variable, and $p_k \coloneqq \mathbb P(N \geq k).$
Then, taking a single sample $N \sim \text{Geom}(p),$ an unbiased estimator of $S$ is given as $S_N$, where $S_n$ is defined for any $n \in \mathbb N$ and $x \in \R$ as
\begin{equation} \label{eq:russian_roulette}
    S_n(x) \coloneqq \sum_{j=1}^N \frac{(-1)^{j+1}}{j} \frac{\text{tr}\left( \jac g_\layeridx(x)^j\right)}{p_j}
\end{equation}
for any $x \in \R^d$.
We will study the variance of $S_N$ in this section.\footnote{\citet{chen2019residual} additionally approximate $\text{tr} \left(\jac g_\layeridx(x)^j\right)$ by the Hutchinson's trace estimator $ v^T \jac g_\layeridx(x)^j v$ for $v \sim \mathcal N(0, I)$.
Since $v$ is independent of $N$, their estimator has strictly higher variance than \eqref{eq:russian_roulette}.}

First, however, define the quantity $\alpha_j(x)$ for $j \in \mathbb N, x \in \R^d$ as 
\begin{equation} \label{eq:alpha}
    \alpha_j(x) \coloneqq \frac{(-1)^{j+1}}{j} \text{tr}\left( \jac g(x)^j\right),
\end{equation}
where we now drop the dependence of $g$ on $\layeridx$.
Then, $S(x) = \sum_{j=1}^\infty \alpha_j(x),$ and $S_N(x) = \sum_{j=1}^N \alpha_j(x) / p_j.$

\subsubsection{What might happen when $\kappa \rightarrow 1$?} \label{sec:kappa_to_one}

We begin with an informal discussion on the variance of $S_N$ as $\kappa \rightarrow 1$.
First of all we know that, as $\kappa \rightarrow 1$, the mapping $f^{-1}$ gets arbitrarily close to a non-invertible mapping: consider e.g.\ $g(x) = -\kappa x$, then $f^{-1} = (1-\kappa)\text{Id} \rightarrow 0$ as $\kappa \rightarrow 1$.
This near non-invertibility has implications for the speed of convergence of both $S(x)$ and its gradient,\footnote{With respect to the flow parameters $\theta$} as noted in these two results from \citet{behrmann2019invertible}:
\begin{enumerate}
    \item \textbf{Theorem 3:} $\left| \sum_{j=1}^n \alpha_j(x) - \log \det\left(I + \jac g(x)\right) \right| \leq -d \left(\log (1 - \kappa) + \sum_{j=1}^n \frac{\kappa^j}{j} \right)$,
    \item \textbf{Theorem 4:} $\left\| \nabla_\theta \left( \alpha_j(x) - \log \det\left(I + \jac g(x)\right) \right) \right\|_\infty = \mathcal O\left(\kappa^n\right)$. 
\end{enumerate}
We can see that both bounds become very loose as $\kappa \rightarrow 1$, implying we cannot guarantee the fast convergence of either series.
It then follows that we cannot invoke the results from \citet{rhee2015unbiased} and \citet{beatson2019efficient} to argue that the variance of the Russian roulette estimator $S_N$ will be small. 
Indeed, in the next section, we will look at a specific example where this variance becomes \emph{infinite}.

\subsubsection{A Specific Example of Infinite Variance}

Now consider the case where $d = 1$.
We will show that when $\kappa^2 > 1-p$, there is a set of $x$ having positive Lebesgue measure such that $S_N(x)$ from \eqref{eq:russian_roulette} has infinite variance.

We note that here we have $\text{tr}\left(\jac g(x)^j \right) = (g'(x))^j$ for any $j \in \mathbb N$.
We can thus rewrite $\alpha_j$ from \eqref{eq:alpha} as 
\begin{equation} \label{eq:alpha_1d}
    \alpha_j(x) \coloneqq \frac{(-1)^{j+1}}{j} (g'(x))^j.
\end{equation}
Also recall that $N \sim \text{Geom}(p)$ and $p_j \coloneqq \mathbb P(N \geq j)$ for all $j \in \mathbb N$.

\begin{proposition} \label{prop:S_finite_exp}
For any $x \in \R$ and random variable $N$ satisfying $\supp N = \mathbb N$, $S_N(x)$ has finite expectation if $\kappa < 1$.
\begin{proof}
    Refer to \citet[Proposition~A.1]{lyne2015russian}.
\end{proof}

\end{proposition}

\begin{proposition} \label{prop:variance_bound}
    Under the same conditions as \autoref{prop:S_finite_exp},
    \[
        \Var S_N(x) \geq \lim_{n \rightarrow \infty} 2 \sum_{j=1}^n \alpha_j(x) S_{j-1}(x) - \mathbb E[S_N(x)]^2.
    \]
\begin{proof}
    This proof is taken from \citet[Proposition~A.2]{lyne2015russian}; we mostly rewrite the proof but adapt it to our specific setting and notation.
    Note that we will drop the dependence of $S_j$ and $\alpha_j$ on $x$ throughout the proof. 

    We know from \autoref{prop:S_finite_exp} that $\E [S_N(x)]$ is finite.
    Thus we will simply lower-bound $\E [S_N(x)^2]$.
    
    We will first use induction to show the following holds for any $n \in \mathbb N$: 
    \begin{equation} \label{eq:induction_to_prove}
        \sum_{j=1}^n S_j^2 (p_j - p_{j+1}) = \alpha_1^2 + \sum_{j=2}^n \frac{\alpha_j^2}{p_j} + 2 \sum_{j=2}^n \alpha_j S_{j-1} - S_n^2 p_{n+1}.
    \end{equation}
    The base case is
    \[
        S_1^2(p_1 - p_2) = \frac{\alpha_1^2}{p_1^2}p_1 - S_1^2 p_2 = \alpha_1^2 - S_1^2 p_2
    \]
    since $p_1 = 1$.
    Now, assume \eqref{eq:induction_to_prove} holds for some $m \in \mathbb N$. Then, for $n = m+1$,
    \begin{align}
        \sum_{j=1}^{m+1} S_j^2 (p_j - p_{j+1}) &= \sum_{j=1}^m S_j^2(p_j - p_{j+1}) + S_{m+1}^2(p_{m+1} - p_{m+2}) \nonumber \\
        &= \alpha_1^2 + \sum_{j=2}^m \frac{\alpha_j^2}{p_j} + 2 \sum_{j=2}^m \alpha_j S_{j-1} - S_m^2 p_{m+1} \label{eq:induction_step} \\
        &\qquad + S_{m+1}^2(p_{m+1} - p_{m+2}) \nonumber
    \end{align}
    by the inductive hypothesis.
    We also have 
    \begin{align*}
        p_{m+1}(S_m^2 - S_{m+1}^2) &= p_{m+1}(S_m - S_{m+1}) (S_m + S_{m+1}) \\
        &= p_{m+1} \frac{\alpha_{m+1}}{p_{m+1}} \left(2 S_m + \frac{\alpha_{m+1}}{p_{m+1}}\right) \\
        &= \frac{\alpha_{m+1}^2}{p_{m+1}} + 2 \alpha_{m+1} S_m.
    \end{align*}
    Substituting this result into \eqref{eq:induction_step} completes the induction and proves \eqref{eq:induction_to_prove} for all $n \in \mathbb N$.
    
    Now, by Jensen's inequality, 
    \[  
        S_n^2 = \left(\sum_{j=1}^n \frac{p_j\frac{\alpha_j}{p_j}}{p_j} \right)^2 \leq \frac{\sum_{j=1}^n \frac{\alpha_j^2}{p_j}}{\sum_{j=1}^n p_j}.
    \]
    This implies
    \begin{align*}
        p_{n+1} S_n^2 \leq p_n S_n^2 \leq \frac{p_n}{\sum_{j=1}^n p_j} \sum_{j=1}^n \frac{\alpha_j^2}{p_j} \leq \sum_{j=1}^n \frac{\alpha_j^2}{p_j}
    \end{align*}
    since $(p_n)$ is a positive sequence.
    
    This finally implies the following lower bound for any $n \in \mathbb N$:
    \begin{align*}
        \sum_{j=1}^n S_j^2 \mathbb P(N = j) &= \sum_{j=1}^n S_j^2 (p_j - p_{j+1}) \\
        &= \alpha_1^2 + \sum_{j=2}^n \frac{\alpha_j^2}{p_j} + 2 \sum_{j=2}^n \alpha_j S_{j-1} - S_n^2 p_{n+1} \\
        &\geq \alpha_1^2 + \sum_{j=2}^n \frac{\alpha_j^2}{p_j} + 2 \sum_{j=2}^n \alpha_j S_{j-1} - \sum_{j=1}^n \frac{\alpha_j^2}{p_j} \\
        &= \alpha_1^2 (1 - p_1^{-1}) + 2 \sum_{j=2}^n \alpha_j S_{j-1} \\
        &= 2 \sum_{j=2}^n \alpha_j S_{j-1},
    \end{align*}
    where the final line follows because $p_1 = 1$.
    
    Since $\E[S_N^2] = \lim_{n\rightarrow\infty} \sum_{j=1}^n S_j^2 \mathbb P(N = j)$, the proof is complete.
\end{proof}

\end{proposition}

We are about ready to prove the main result but require one more auxiliary result first.
\begin{proposition} \label{prop:limit_sum}
    Suppose $|b| > 1$. 
    Then, 
    \[
        \lim_{n\rightarrow\infty} \frac{n}{b^n} \sum_{j=1}^{n-1} \frac{b^j}{j} = \frac{1}{b-1}.
    \]
\begin{proof}
    We will first show that the limit exists, and then show that it equals $(b-1)^{-1}$.
    Let
    \[
        c_n = \frac{n}{b^n} \sum_{j=1}^{n-1} \frac{b^j}{j}.
    \]
    We can rewrite this as follows:
    \[
        c_n = \sum_{j=1}^{n-1} \frac{n}{b^{n-j} j} = \sum_{j=1}^{n-1} \frac{n}{b^j (n-j)} = \sum_{j=1}^{n-1} \frac{1}{b^j} + \sum_{j=1}^{n-1}\frac{j}{b^j(n-j)}.
    \]
    Since $b > 1$, the first sum is a convergent geometric series as $n \rightarrow \infty$.
    We can decompose the second sum into its positive and negative terms: 
    \[
        \sum_{j=1}^{n-1} \frac{j}{b^j(n-j)} = \sum_{j \geq 1:b^j > 0}^{n-1} \frac{j}{b^j(n-j)} + \sum_{j \geq 1:b^j < 0}^{n-1} \frac{j}{b^j(n-j)} \equiv \circled{1}_n + \circled{2}_n. 
    \]
    We can see, for all $n \in \mathbb N$,
    \[
    \circled{1}_n \geq - \sum_{j=1}^{n-1} \frac{j}{|b|^j(n-j)} \qquad \text{and} \qquad 
    \circled{2}_n \leq \sum_{j=1}^{n-1} \frac{j}{|b|^j(n-j)}.
    \]
    Furthermore, for all $j \in \{1, \ldots, n-1\},$ we have 
    \[
        \frac{j}{n-j} \leq j.
    \]
    Now notice that the series $\sum_{j=1}^\infty \frac{j}{|b|^j}$ converges by the ratio test:
    \[
        \lim_{j\rightarrow\infty} \left|\frac{\frac{j+1}{|b|^{j+1}}}{\frac{j}{|b|^j}} \right| = \lim_{j\rightarrow\infty} \frac{j+1}{j|b|} = \frac{1}{|b|} < 1. 
    \]
    This implies the existence of $\lim_{n\rightarrow\infty}\sum_{j=1}^{n-1} \frac{j}{|b|^j (n-j)}$. 
    Since the sequence $(\circled{1}_n)$ (resp.\ $(\circled{2}_n)$) is negative, non-increasing, and bounded below (resp.\ positive, non-decreasing, and bounded above),
    this implies the existence of $\lim_{n\rightarrow \infty} \circled{1}_n$ (resp.\ $\lim_{n\rightarrow \infty} \circled{2}_n$).
    Altogether, this implies the existence of
    \[
        \lim_{n\rightarrow\infty}\left( \sum_{j=1}^{n-1} \frac{1}{b^j} + \sum_{j=1}^{n-1}\frac{j}{b^j(n-j)}\right) = \lim_{n\rightarrow\infty} c_n \eqqcolon c_\infty.
    \]
    
    Now we will determine its precise value. 
    Note the following recurrence for all $n \in \mathbb N$:
    \[
        c_{n+1} = \frac{n+1}{b n} \left(1 + c_n\right).
    \]
    Taking the limit of both sides as $n \rightarrow \infty$ gives
    \[
        c_\infty = \frac{1}{b} (1 + c_\infty).
    \]
    Solving this gives us $c_\infty = \frac{1}{b-1}$, which completes the proof.
\end{proof}
\end{proposition}

\begin{proposition} \label{prop:geom_RR_variance}

Suppose $N \sim \text{Geom}(p)$, $g$ is continuously differentiable, and $1 - p < \kappa^2 < 1$.
Then 
\[
    \{x \in \R \mid \Var S_N(x) = \infty\}
\]
has positive Lebesgue measure.

\begin{proof}

    From \autoref{prop:variance_bound}, for a given $x \in \R$, we can see that showing $\sum_{n=2}^\infty \alpha_n(x) S_{n-1}(x)$ diverges is sufficient to prove $\Var S_N(x)$ is infinite.
    
    Consider using the ratio test to assess the convergence of the above series, with terms defined as $a_n(x) \coloneqq \alpha_n(x) S_{n-1}(x)$.
    We have the following for any $n \geq 2$:
    \begin{align*}
        \left|\frac{a_{n+1}(x)}{a_n(x)} \right| &= \left|\frac{\alpha_{n+1}(x) S_n(x)}{\alpha_n(x) S_{n-1}(x)}\right| \\
        &= \frac{\frac{|(g'(x))^{n+1}|}{n+1}}{\frac{|(g'(x))^n|}{n}} \cdot \left|\frac{\sum_{j=1}^n \frac{\alpha_j(x)}{p_j}}{\sum_{j=1}^{n-1} \frac{\alpha_j(x)}{p_j}}\right| \\
        &= \frac{n |g'(x)|}{n + 1} \cdot\left| \frac{(-1)^{n+1} \cdot (g'(x))^n}{n p_n}\left(\sum_{j=1}^{n-1} \frac{(-1)^{j+1}\cdot (g'(x))^j}{j p_j} \right)^{-1} + 1 \right|.
    \end{align*}
    
    Recall $p_j = (1-p)^{j-1} \equiv q^{j-1}$.
    Then, writing $b = - \frac{g'(x)}{q}$, we have 
    \[
        \frac{(-1)^{n+1} \cdot (g'(x))^n}{n p_n}\left(\sum_{j=1}^{n-1} \frac{(-1)^{j+1}\cdot (g'(x))^j}{j p_j} \right)^{-1} = \frac{1}{n} b^n \left(\sum_{j=1}^{n-1} \frac{1}{j} b^j\right)^{-1}.
    \]
    
    Now let us assume that $|g'(x)|^2 > q$.
    We can see that $|g'(x)|^2 > q \implies |g'(x)| > q$ since $q \in (0, 1)$, which then entails $|b| > 1$.
    Therefore, by \autoref{prop:limit_sum},
    \[
        \lim_{n\rightarrow\infty} \frac{n}{b^n} \sum_{j=1}^{n-1} \frac{b^j}{j} = \frac{1}{b-1}.
    \]
    This then implies
    \begin{align*}
        \lim_{n\rightarrow\infty} \left| \frac{a_{n+1}(x)}{a_n(x)} \right| &= \lim_{n\rightarrow\infty} \frac{n |g'(x)|}{n+1} \left|\frac{1}{\frac{n}{b^n}\sum_{j=1}^{n-1}\frac{b^j}{j}}  + 1\right| \\
        &= |g'(x)| \left|\frac{1}{\frac{1}{b-1}} + 1\right| = \frac{|g'(x)|^2}{q} > 1
    \end{align*}
    since we have assumed that $|g'(x)|^2 > q$.
    Thus, for all $x$ in the set 
    \[
        V_{g,q} \coloneqq \{x \in \R \mid |g'(x)|^2 > q\},
    \]
    the series $\sum_{n=2}^\infty \alpha_n(x) S_{n-1}(x)$ diverges by the ratio test.
    This means that $\Var S_N(x) = \infty$ for all $x \in V_{g, q}$. 
    
    Finally, we will prove the set $\{x \in \R \mid \Var S_N(x) = \infty\}$ has positive Lebesgue measure. 
    Recall that $\Lip g = \kappa$, which directly implies $\sup_{x\in\R} |g'(x)| = \kappa$ from \autoref{prop:lip-const-op-norm-expression} and thus $\sup_{x\in\R} |g'(x)|^2 = \kappa^2$.
    Then, since $\kappa^2 > q$, there exists $x_0 \in \R$ such that $|g'(x_0)|^2 \in (q, \kappa^2)$.
    By the continuity of $|g'|$, there is open ball of nonzero radius around $x_0$, denoted $\mathcal B(x_0)$, such that $|g'(x)| > q$ for all $x \in \mathcal B(x_0).$
    Since $\mathcal B(x_0)$ is open and non-empty, it has positive Lebesgue measure.
    The inclusions
    \[
        \mathcal B(x_0) \subseteq V_{g,q} \subseteq \{x \in \R \mid \Var S_N(x) = \infty\}
    \]
    thus conclude the proof.
    
\end{proof}

\end{proposition}

\subsubsection{Discussion}

\paragraph{Changing $p$ as $\kappa$ increases} An obvious strategy to avoid satisfying the conditions of \autoref{prop:geom_RR_variance} is to set $p$ such that $1 - \kappa^2 > p$.
However, lowering $p$ in this way incurs additional computational cost: the average number of iterations per training step is equal to $p^{-1}$, or is lower-bounded by  $(1-\kappa^2)^{-1}$ if $p < 1 - \kappa^2$.
Thus, if we send $\kappa \rightarrow 1$ to mitigate the bi-Lipschitz constraint \eqref{eq:bilip-constraint}, we will either incur an infinite computational cost or run the risk of encountering infinite variance. 

\paragraph{Higher dimensions} Although \autoref{prop:geom_RR_variance} only applies for $d=1$, it is conceivable that similar results can be derived for $d > 1$, especially when considering the discussion in \autoref{sec:kappa_to_one}.
We leave a deeper investigation for future work.

\subsection{Density of a \modelacr}
\label{sec:density-of-a-cif}

We make precise our heuristic derivation of the density \eqref{eq:cif-joint-density} via the following result.

\begin{proposition}
    Suppose $\noisespace, \dataspace \subseteq \R^d$ are open, and that $\Fwdmap(\cdot ; \idx) : \noisespace \to \dataspace$ is a continuously differentiable bijection with everywhere invertible Jacobian for each $\idx \in \idxspace$.
    Under the generative model \eqref{eq:our-model-stacked}, $(\datavar, \idxvar)$ has joint density
    \[
        \prior(\Fwdmap^{-1}(\data; \idx)) \, \idxdist(\idx|\Fwdmap^{-1}(\data; \idx)) \, \abs{\det \jac \Fwdmap^{-1}(\data; \idx)}.
    \]
\end{proposition}

\begin{proof}
    Suppose $h : \dataspace \times \idxspace \to \R$ is a bounded measurable test function. Then
    \begin{align*}
        \E[h(\datavar, \idxvar)] &= \E[h(\Fwdmap(\noisevar; \idxvar), \idxvar)] \\
        &=\int \left[\int h(\Fwdmap(\noise; \idx), \idx) \, p_\noisevar(\noise) \, \idxcond(\idx|\noise) \diff \noise \right] \diff \idx \\
        &= \int h(\data, \idx) \, \prior(\Fwdmap^{-1}(\data; \idx)) \, \idxcond(\idx|\Fwdmap^{-1}(\data; \idx)) \,  \abs{\det \jac \Fwdmap^{-1}(\data; \idx)} \diff \noise \diff \idx,
    \end{align*}
    where in the third line we substitute $\data \coloneqq \Fwdmap(\noise; \idx)$ on the inner integral, which is valid by Theorem 17.2 of \citet{billingsley2008probability}.
    Now for $A \subseteq \dataspace \times \idxspace$, let $h \coloneqq \ind_A$. It follows that
    \[
        \P((X, U) \in A) = \E[\ind_A(X, U)] = \int_A \prior(\Fwdmap^{-1}(\data; \idx)) \, \idxcond(\idx|\Fwdmap^{-1}(\data; \idx)) \,  \abs{\det \jac \Fwdmap^{-1}(\data; \idx)} \diff \noise \diff \idx,
    \]
    which gives the result since $A$ was arbitrary.
\end{proof}

\subsection{Our Approximate Posterior Does Not Sacrifice Generality} \label{sec:approx-postrior-has-correct-form}

The following result shows that our parameterisation of the approximate posterior $q_{\idxvar_{1:\numlayers}|\datavar}$ in \eqref{eq:approx-posterior-form} does not lose generality.
In particular, provided each $q_{\idxvar_{\ell}|\noisevar_\ell}$ is sufficiently expressive, we can always recover the exact posterior.

\begin{proposition}
    Under the generative model \eqref{eq:our-model-stacked}, the posterior factors like
    \[
        p_{\idxvar_{1:\numlayers}|\datavar}(\idx_{1:\numlayers}|\data) = \prod_{\layeridx=1}^\numlayers p_{\idxvar_\layeridx|\noisevar_\layeridx}(\idx_\layeridx|\noise_\layeridx),
    \]
    where $\noise_\numlayers \coloneqq \data$ and $\noise_{\layeridx} \coloneqq \Fwdmap^{-1}_{\layeridx+1}(\noise_{\layeridx+1}; \idx_{\layeridx+1})$ for $\layeridx \in \{1, \ldots, \numlayers-1\}$.
\end{proposition}

\begin{proof}
    Writing $p_{\idxvar_{1:\numlayers}|\datavar}$ autoregressively gives
    \[
        p_{\idxvar_{1:\numlayers}|\datavar}(\idx_{1:\numlayers}|\data) = \prod_{\layeridx=1}^\numlayers p_{\idxvar_\layeridx|\idxvar_{\layeridx+1:\numlayers},\datavar}(\idx_\layeridx|\idx_{\ell+1:\numlayers},\data).
    \]
    But now it is clear from the generative model \eqref{eq:our-model-stacked} that $\idxvar_\ell$ is conditionally independent of $(\idxvar_{\ell+1:\numlayers}, \datavar)$ given $\noisevar_\ell$, and as such
    \[
         p_{\idxvar_\layeridx|\idxvar_{\layeridx+1:\numlayers}, \datavar}(\idx_\layeridx|\idx_{\ell+1:\numlayers}, \data) = p_{\idxvar_\layeridx|\noisevar_\layeridx}(\idx_\layeridx|\noise_\layeridx).
    \]
    Substituting this into the above expression then gives the result.
\end{proof}

\subsection{Conditions for a \modelacr\ to Outperform an Underlying Normalising Flow} \label{sec:conditions_to_outperform}

For this result, the components of our model are assumed to be parameterised by $\theta \in \Theta$, which we will indicate by by $\Fwdmap_\theta$, $\idxcond^\theta$, and $\approxidxpost^\theta$.
We will also use $\theta$ to indicate quantities that result from the choice of parameters $\theta$ (e.g.\ $\modelmeas^\theta$ for the distribution obtained), and will denote by $\elbo^\theta$ the corresponding ELBO \eqref{eq:elbo}.

\cifnoworse*

\begin{proof}
    Observe from \eqref{eq:cif-joint-density} that
    \[
        \dataidxjoint^\phi(\data, \idx) = \prior(\fwdmap^{-1}(\data)) \, \abs{\det \jac \fwdmap^{-1}(\data)} \,\idxcond(\idx|\fwdmap^{-1}(\data)).
    \]
    It then follows from \eqref{eq:our-density} that, under $\phi$, the model has density
    \[
        \model^\phi(\data) = \prior(\fwdmap^{-1}(\data)) \, \abs{\det \jac \fwdmap^{-1}(\data)} \, \underbrace{\int \idxcond(\idx|\fwdmap^{-1}(\data))  \, \diff \idx}_{=1}
    \]
    which is exactly the density of the normalising flow $\pushforward{\fwdmap}{\priormeas}$.
    We also obtain the posterior
    \begin{align*}
        \idxpost^\phi(\idx|\data) &= \frac{\dataidxjoint^\phi(\data, \idx)}{\model^\phi(\data)} \\
        &= \idxcond(\idx|\fwdmap^{-1}(\data)) \\
        &= r(\idx).
    \end{align*}
    Since each $\approxidxpost^\phi(\cdot|\data) = r(\cdot)$ also, it follows that $\elbo^\phi$ is tight, so that $\elbo^\phi(\data) = \log \model^\phi(\data)$ for all $\data \in \dataspace$.

    Now suppose some $\theta \in \Theta$ has
    \[
        \E_{\data \sim \targetmeas}[\elbo^{\theta}(\data)] \geq \E_{\data \sim \targetmeas}[\elbo^{\phi}(\data)].
    \]
    It follows that
    \[
        \E_{\data \sim \targetmeas}[\log \model^{\theta}(\data)]
            \geq \E_{\data \sim \targetmeas}[\elbo^{\phi}(\data)]
            = \E_{\data \sim \targetmeas}[\log \model^\phi(\data)].
    \]
    Subtracting $\E_{\data \sim \targetmeas}[\log \target(\data)]$ from both sides and negating gives
    \[
        \KL{\targetmeas}{\modelmeas^\theta} \leq \KL{\targetmeas}{\modelmeas^\phi} = \KL{\targetmeas}{\pushforward{\fwdmap}{\priormeas}}.
    \]
\end{proof}

\subsection{\modelacr s Can Learn Target Supports Exactly} \label{sec:support-correction-result}

In this section we give necessary and sufficient conditions for a \modelacr\ to learn the support of a target distribution exactly, without needing changes to $\Fwdmap$.
However, our argument applies more generally and does not make specific use of the bijective structure of $\Fwdmap$.
To make this clear, we formulate our result here in terms of a generalisation of the model \eqref{eq:our-model}.
In particular, we will take $\modelmeas$ as the marginal in $\datavar$ of
\begin{equation} \label{eq:generalised-cif}
    \noisevar \sim \priormeas, \quad \idxvar \sim \idxcondmeas(\cdot|\noisevar), \quad \datavar \coloneqq G(\noisevar, \idxvar),
\end{equation}
where $G : \noisespace \times \idxspace \to \dataspace$.
We will assume that
\begin{itemize}
    \item $\noisespace \subseteq \R^{d_\noisespace}$, $\idxspace \subseteq \R^{d_\idxspace}$, and $\dataspace \subseteq \R^{d_\dataspace}$ are equipped with the subspace topology;
    \item $\priormeas$ and each $\idxcondmeas(\cdot|\noise)$ are Borel probability measures on $\noisespace$ and $\idxspace$ respectively;
    \item $G$ is continuous with respect to the product topology $\noisespace \times \idxspace$.
\end{itemize}
We then have the following formula for $\supp \modelmeas$:

\begin{lemma} \label{lem:support-expression}
    Under the model \eqref{eq:generalised-cif},
    \[
        \supp \modelmeas = \closure{\bigcup_{\noise \in \supp \priormeas} G(\{\noise\} \times \supp \idxcondmeas(\cdot|\noise))}.
    \]
\end{lemma}

\begin{proof}
    Denote the joint distribution of $(\noisevar, \idxvar)$ by $P_{\noisevar, \idxvar}$.
    Observe from \autoref{lem:pushforward-support} that
    \[
        \supp \modelmeas = \closure{G(\supp P_{\noisevar, \idxvar})}.
    \]
    Let
    \[
        B \coloneqq \bigcup_{\noise \in \supp \noisevar} \{\noise\} \times \supp \idxcondmeas(\cdot|\noise).
    \]
    The result follows if we can show that
    \[
        \supp P_{\noisevar, \idxvar} = \closure{B},
    \]
    since $\closure{G(\closure{B})} = \closure{G(B)}$ because $G$ is continuous.

    We first show that $\supp P_{\noisevar, \idxvar} \supseteq \closure{B}$.
    Suppose $(\noise, \idx) \in B$, and let $N_{(\noise, \idx)} \subseteq \noisespace \times \idxspace$ be an open set containing $(\noise, \idx)$.
    Then there exists open $N_\noise$ and $N_\idx$ containing $\noise$ and $\idx$ respectively such that $N_\noise \times N_\idx \subseteq N_{(\noise, \idx)}$, since the open rectangles form a base for the product topology.
    It follows that
    \begin{align*}
        P_{\noisevar, \idxvar}(N_{(\noise, \idx)}) &\geq P_{\noisevar, \idxvar}(N_\noise \times N_\idx) \\
        &= \int_{N_\noise} \idxcondmeas(N_\idx|\noise') \, P_\noisevar(\diff \noise') \\
        &> 0,
    \end{align*}
    since by the definition of $B$ we have $\priormeas(N_\noise) > 0$ and $\idxcondmeas(N_\idx|\noise) > 0$ for each $\idx \in N_\noise$. From this we have $\supp P_{\noisevar, \idxvar} \supseteq B$, and taking the closure of each side gives $\supp P_{\noisevar, \idxvar} \supseteq \closure{B}$.

    In the other direction, suppose that $(\noise, \idx) \not \in \closure{B}$. Then there exist open sets $N_\noise$ and $N_\idx$ containing $\noise$ and $\idx$ respectively such that
    \[
        (N_\noise \times N_\idx) \cap \closure{B} = \emptyset.
    \]
    By the definition of $B$, it follows that if $(\noise', \idx') \in N_\noise \times N_\idx$ and $\noise' \in \supp \priormeas$, then $\idx' \not \in \supp \idxcondmeas(\cdot|\noise')$. Otherwise stated, if $\noise' \in N_\noise \cap \supp \priormeas$, then
    \[
        N_\idx \cap \supp \idxcondmeas(\cdot|\noise') = \emptyset.
    \]
    Thus
    \begin{align*}
        P_{\noisevar, \idxvar}(N_\noise \times N_\idx) &= \int_{N_\noise} \left[\int_{N_\idx} \idxcondmeas(\diff \idx'|\noise')\right] \, \priormeas(\diff \noise') \\
        &= \int_{N_\noise \cap \supp \priormeas} \left[\int_{N_\idx \cap \supp \idxcondmeas(\cdot|\noise')} \idxcondmeas(\diff \idx'|\noise') \right] \, \priormeas(\diff \noise') \\
        &= 0,
    \end{align*}
    where the second line follows from \autoref{lem:support-all-that-matters}. Consequently $(\noise, \idx) \not \in \supp P_{\noisevar, \idxvar}$, which gives $\supp P_{\noisevar, \idxvar} \subseteq \closure{B}$.
\end{proof}

We now give necessary and sufficient conditions for the model \eqref{eq:generalised-cif} to learn a given target support exactly.

\begin{proposition}
    Suppose $\targetmeas(\boundary \supp \targetmeas) = 0$ and that
    \begin{equation} \label{eq:covering-assumption}
        \closure{G(\supp \priormeas \times \idxspace)} \supseteq \supp \targetmeas.
    \end{equation}
    Then there exists $\idxcondmeas$ such that $\supp \modelmeas = \supp \targetmeas$ if and only if, for all $\noise \in \supp \priormeas$, there exists $\idx \in \idxspace$ with
    \[
        G(\noise, \idx) \in \supp \targetmeas.
    \]
\end{proposition}

\begin{proof}
    ($\Rightarrow$) Choose $\idxcondmeas$ such that $\supp \modelmeas = \supp \targetmeas$. \autoref{lem:support-expression} gives
    \[
            \bigcup_{\noise \in \supp \priormeas} G(\{\noise\} \times \supp \idxcondmeas(\cdot|\noise)) \subseteq \supp \targetmeas.
    \]
    Suppose $\noise \in \supp \priormeas$.
    Then for indeed all $\idx \in \supp \idxcondmeas(\cdot|\noise)$ we must have $G(\noise, \idx) \in \supp \targetmeas$, which proves this direction since $\supp \idxcondmeas(\cdot|\noise) \neq 0$ by \autoref{lem:support-all-that-matters}.

    ($\Leftarrow$) For $\noise \in \supp \priormeas$, let
    \begin{equation} \label{eq:idxcondmeas-support}
        A_\noise \coloneqq \{\idx \in \idxspace : G(\noise, \idx) \in \interior(\supp \targetmeas)\},
    \end{equation}
    where $\interior$ denotes the \emph{interior} operator.
    If $A_\noise = \emptyset$, define $\idxcondmeas(\cdot|\noise)$ to be Dirac on some $\idx$ such that $G(\noise, \idx) \in \supp \targetmeas$, which exists by assumption.
    Otherwise, we let $\idxcondmeas(\cdot|\noise)$ be a probability measure with support $\closure{A_\noise}$.
    To show that such a measure exists, observe that $A_\noise$ is open since $G$ is continuous.
    Since $\idxspace$ is separable, we can therefore write
    \[
        A_\noise = \bigcup_{n=1}^\infty B_n
    \]
    for a countable collection of open sets $B_n \subseteq A_\noise$.
    We can then define a probability measure $\mu$ by
    \[
        \mu(C) \coloneqq \sum_{n=1}^\infty 2^{-n} \, \ind(C \cap B_n \neq \emptyset)
    \]
    for measurable $C \subseteq \idxspace$.
    Since $A_\noise \neq \emptyset$, it is straightforward to see that this is a probability measure with $\mu(A_\noise) = 1$.
    Consequently $\supp \mu = \closure{A_\noise}$ by \autoref{lem:support-closed}, since $A_\noise$ is open and $\supp \mu$ is the smallest closed set with $\mu$-probability 1.

    We show this construction gives $\supp \modelmeas \subseteq \supp \targetmeas$.
    To this end, we first prove that if $\noise \in \supp \priormeas$ and $\idx \in \supp \idxcondmeas(\cdot|\noise)$ then
    \[
        G(\noise, \idx) \in \supp \targetmeas.
    \]
    If $A_\noise = \emptyset$ this is immediate.
    Otherwise, since $\supp \idxcondmeas(\cdot|\noise) = \closure{A_\noise}$, there exists $(\idx_n) \subseteq A_\noise$ such that $\idx_n \to \idx$.
    By \eqref{eq:idxcondmeas-support}, each $G(\noise, \idx_n) \in \supp \targetmeas$.
    By continuity we then have
    \[
        G(\noise, \idx_n) \to G(\noise, \idx) \in \supp \targetmeas
    \]
    since $\supp \targetmeas$ is closed.
    It follows that
    \[
        \bigcup_{\noise \in \supp \priormeas} G(\{\noise\} \times \supp \idxcondmeas(\cdot|\noise)) \subseteq \supp \targetmeas,
    \]
    which gives $\supp \modelmeas \subseteq \supp \targetmeas$ from \autoref{lem:support-expression} since $\supp \targetmeas$ is closed.

    We now show $\supp \modelmeas \supseteq \supp \targetmeas$.
    Since $\targetmeas(\boundary \supp \targetmeas) = 0$ we have
    \[
        \supp \targetmeas = \closure{\interior(\supp \targetmeas)}
    \]
    by \autoref{lem:support-closed}, so that $\supp \modelmeas \supseteq \supp \targetmeas$ if $\supp \modelmeas \supseteq \interior(\supp \targetmeas)$.
    Now suppose $\data \in \interior(\supp \targetmeas)$.
    Then there exists $(\noise_n) \subseteq \supp\priormeas$ and $(\idx_n) \subseteq \idxspace$ such that $G(\noise_n, \idx_n) \to \data$ by \eqref{eq:covering-assumption}.
    But then we must have $G(\noise_n, \idx_n) \in \interior(\supp \targetmeas)$ for $n$ large enough because $\data$ lies in the interior.
    Consequently, for $n$ large enough,
    \[
        \idx_n \in A_{\noise_n} \subseteq \supp \idxcondmeas(\cdot|\noise_n)
    \]
    and hence $G(\noise_n, \idx_n) \in \supp \modelmeas$ by \autoref{lem:support-expression}.
    This means $\data \in \supp \modelmeas$ since $\supp \modelmeas$ is closed.
\end{proof}

The following proposition then gives a straightforward condition under which it is additionally possible to recover the \emph{target} exactly (i.e.\ not just its support).
In our experiments we do not enforce this condition explicitly.
However, since we learn the parameters of $G$ here, we can expect our model will approximate this behaviour if doing so produces a better density estimator.

\begin{proposition}
    If $G(\noise, \cdot)$ is surjective for each $\noise \in \noisespace$, then there exists $\idxcondmeas$ such that $\modelmeas = \targetmeas$.
\end{proposition}

\begin{proof}
    Fix $\noise \in \noisespace$.
    Surjectivity of $G(\noise, \cdot)$ means that, for $\data \in \dataspace$, there exists $\idx \in \idxspace$ such that $G(\noise, \idx) = \data$.
    Thus we can define $H_\noise : \dataspace \to \idxspace$ such that
    \[
        G(\noise, H_\noise(\data)) = \data
    \]
    for all $\data \in \dataspace$.
    We then define each
    \[
        \idxcondmeas(\cdot|\noise) \coloneqq \pushforward{H_\noise}{\targetmeas}.
    \]
    From this it follows that $\modelmeas = \targetmeas$. For, letting $B \subseteq \dataspace$ be measurable,
    \begin{align*}
        \modelmeas(B) &= \int_{G^{-1}(B)} \, \idxcondmeas(\diff \idx| \noise) \, \priormeas(\diff \noise) \\
        &= \int \left[ \int \ind_B(G(\noise, \idx)) \, \pushforward{H_\noise}{\targetmeas}(\diff u) \right] \, \priormeas(\diff \noise) \\
        &= \int \left[ \int \ind_B(G(\noise, H_\noise(\data))) \, \targetmeas(\diff \data) \right] \, \priormeas(\diff \noise) \\
        &= \int \left[ \int \ind_B(\data) \, \targetmeas(\diff \data) \right] \, \priormeas(\diff \noise) \\
        &= \targetmeas(B),
    \end{align*}
    which gives the result.
\end{proof}

\section{Experimental Details} \label{sec:full-experimental-details}

Our choices \eqref{eq:our-idxcond} and \eqref{eq:our-approx-posterior} required parameterising $\scale$, $\trans$, $\mean^p$, $\Sigma^p$, $\mean^q$, and $\Sigma^q$.
Since these terms are naturally paired, at each layer of our model we set
\begin{align*}
    [\scale(\idx), \trans(\idx)] &\coloneqq \NN_{\Fwdmap}(\idx), \\
    [\mean^p(\noise), \varsigma^p(\noise)] &\coloneqq \NN_{p}(\noise), \\ \Sigma^p(\noise) &\coloneqq \diag(e^{\varsigma^p(\noise)}), \\
    [\mean^q(\data), \varsigma^q(\data)] &\coloneqq \NN_{q}(\data), \\ \Sigma^q(\data) &\coloneqq \diag(e^{\varsigma^q(\data)}),
\end{align*}
where $\NN$ denotes a separate neural network and $\varsigma^p(\noise), \varsigma^q(\data) \in \R^d$.

In all experiments we trained our models to maximise either the log-likelihood (for the baseline flows) or the ELBO (for the \modelacr s) using the ADAM optimiser \citep{kingma2014adam} with default hyperparameters and no weight decay.
The ELBO was estimated using a single sample per datapoint (i.e.\ a single call to \autoref{alg:ELBO}).
We used a held-out validation set and trained each model until its validation score stopped improving, except for the NSF tabular data experiments where we train for a fixed number of epochs as specified in \citet{durkan2019neural}.
After training, we used validation performance to select the best parameters found during training for use at test time (again except for the NSF experiments, where we just test with the final model).
Both validation and test scores were computed using the exact log-likelihood for the baseline and the importance sampling estimate \eqref{eq:IS-estimator} for the \modelacr s, with $m = 5$ samples for validation and $m = 100$ for testing. 

\subsection{Tabular Data Experiments}\label{sec:UCI_appendix}

Following \citet{papamakarios2017masked}, we experimented with the POWER, GAS, HEPMASS, and MINIBOONE datasets from the UCI repository \citep{Bache+Lichman:2013}, as well as a dataset of $8 \times 8$ image patches extracted from the BSDS300 dataset \cite{martin2001database}.
We preprocessed these datasets identically to \citet{papamakarios2017masked}, and used the same train/validation/test splits.
For all \modelacr-ResFlow models, we used a batch size of 1000 and a learning rate of $10^{-3}$.
For the MAF experiments, we used a batch size of 1000 and a learning rate of $10^{-3}$, except for BSDS300 where we used a learning rate of $10^{-4}$ to control the instability of the baseline.
For the NSF experiments, we used batch sizes and learning rates as dictated by \citet[Table~5]{durkan2019neural}, along with their cosine learning rate annealing scheme.

Also, for all \modelacr\ models, each $\idxvar_\ell$ had the same dimension $d_\idxspace$, which we took to be roughly a quarter of the dimensionality of the data (except in \autoref{sec:ablating-fwdmap} for which $d_\idxspace = d_\dataspace$).
In particular, we set $d_\idxspace \coloneqq 2$ for POWER and GAS, $d_\idxspace \coloneqq 5$ for HEPMASS, $d_\idxspace \coloneqq 10$ for MINIBOONE, and $d_\idxspace \coloneqq 15$ for BSDS300.

\subsubsection{Residual Flows} \label{sec:residual-flows-experiments-supp}

The residual blocks in all ResFlow models used multilayer perceptrons (MLPs) with $4$ hidden layers of $128$ hidden units (denoted $4 \times 128$), LipSwish nonlinearities \citep[(10)]{chen2019residual} before each linear layer, and a residual connection from the input to the output.
We did not use any kind of normalisation (e.g. ActNorm or BatchNorm) for these experiments.
For all models we set $\kappa = 0.9$ in \eqref{eq:resflow} to match the value for the 2-D experiments in the codebase of \citet{chen2019residual}.
Other design choices followed \citet{chen2019residual}. In particular:
\begin{itemize}
    \item We always exactly computed several terms at the beginning of the series expansion of the log Jacobian, and then used Russian Roulette sampling \citep{kahn1955use} to estimate the sum of the remaining terms. In particular, at training time we computed 2 exact terms, while at test time we computed 20 exact terms;
    \item We used a geometric distribution with parameter $0.5$ for the number of terms to compute in our Russian Roulette estimators;
    \item We used the Skilling-Hutchinson trace estimator \citep{skilling1989eigenvalues,hutchinson1990stochastic} to estimate the trace in the log Jacobian term;
    \item At both training and test time, we used a single Monte Carlo sample of $(n, v)$ to estimate (6) of \citet{chen2019residual};
\end{itemize}
However, note that for these experiments, for the sake of simplicity, we did not use the memory-saving techniques in (8) and (9) of \citet{chen2019residual}, nor the adaptive power iteration scheme described in their Appendix E.

For $\NN_{\Fwdmap}$, $\NN_{p}$, and $\NN_q$ we used $2 \times 10$ MLPs with $\tanh$ nonlinearities 
These networks were much smaller than $4 \times 128$, and hence the \modelacr-ResFlows had only roughly 1.5-4.5\% more parameters (depending on the dimension of the dataset) than the otherwise identical 10-layer ResFlows, and roughly 10\% of the parameters of the 100-layer ResFlows.

The 100-layer ResFlows were significantly slower to train than the 10-layer models, and for POWER, GAS, and BSDS300 we were forced to stop these before their validation loss had converged.
However, to ensure a fair comparison, we allocated more total computing power to these models than to the 10-layer models, which were terminated properly.
In particular, we trained each 100-layer ResFlow on POWER and GAS for a total of 10 days on a single NVIDIA GeForce GTX 1080 Ti, and on BSDS300 for a total of 7 days.
In contrast, the 10-layer ResFlows converged after around 1 day on POWER, 4.5 days on GAS, and around 3 days on BSDS300.
Likewise, the 10-layer \modelacr-ResFlows converged after around 1 day on POWER, 6 days on GAS, and 2 days on BSDS300.

\subsubsection{Masked Autoregressive Flows} \label{sec:maf_experiments}

\begin{table}
\caption{MAF and \modelacr-MAF parameter configurations for POWER and GAS.}
\label{tab:params_power}
\begin{center}
\begin{small}
\begin{sc}
\begin{tabular}{llllll}
\toprule
 &\multicolumn{1}{c}{Layers ($\numlayers$)} &\multicolumn{1}{c}{Autoregressive network size} &\multicolumn{1}{c}{$\NN_p$ size}&\multicolumn{1}{c}{$\NN_q$ size}&\multicolumn{1}{c}{$\NN_\Fwdmap$ size} \\
\midrule
MAF     &$5$, $10$, $20$ &  $2 \times 100$, $2 \times 200$, $2 \times 400$ & - & - & - \\
\modelacr-MAF &$5$, $10$ & $2 \times 128$ & $2 \times 100, 2 \times 200$ & $2 \times 100, 2 \times 200$ & $2 \times 128$\\
\bottomrule
\end{tabular}
\end{sc}
\end{small}
\end{center}
\end{table}

\begin{table}
\caption{MAF and \modelacr-MAF parameter configurations for HEPMASS and MINIBOONE}\label{tab:params_hepmass}
\begin{center}
\begin{small}
\begin{sc}
\begin{tabular}{llllll}
\toprule
 &\multicolumn{1}{c}{Layers ($\numlayers$)} &\multicolumn{1}{c}{Autoregressive network size} &\multicolumn{1}{c}{$\NN_p$ size}&\multicolumn{1}{c}{$\NN_q$ size}&\multicolumn{1}{c}{$\NN_\Fwdmap$ size} \\
\midrule
MAF     &$5$, $10$, $20$ &  $2 \times 128$, $2 \times 512$, $2 \times 1024$ & - & - &- \\
\modelacr-MAF &$5$, $10$ & $2 \times 128$ & $2 \times 128, 2 \times 512$ & $2 \times 128, 2 \times 512$ & $2 \times 128$\\
\bottomrule
\end{tabular}
\end{sc}
\end{small}
\end{center}
\end{table}

\begin{table}
\caption{MAF and \modelacr-MAF parameter configurations for BSDS300}\label{tab:params_bsds}
\begin{center}
\begin{small}
\begin{sc}
\begin{tabular}{llllll}
\toprule
 &\multicolumn{1}{c}{Layers ($\numlayers$)} &\multicolumn{1}{c}{Autoregressive network size} &\multicolumn{1}{c}{$\NN_p$ size}& \multicolumn{1}{c}{$\NN_q$ size}&\multicolumn{1}{c}{$\NN_\Fwdmap$ size} \\
\midrule
MAF     &$5$, $10$, $20$ &  $2 \times 512$, $2 \times 1024$, $2 \times 2048$ & - &- &- \\
\modelacr-MAF &$5$, $10$ & $2 \times 512$ & $2 \times 128, 2 \times 512$ & $2 \times 128, 2 \times 512$ & $2 \times 128$\\
\bottomrule
\end{tabular}
\end{sc}
\end{small}
\end{center}
\end{table}

The experiment comparing MAF baselines to \modelacr-MAFs was inspired by the experimental setup in \citet{papamakarios2017masked}. 
For each dataset, we specified a set of hyperparameters over which to search for both the baselines and the \modelacr s; these hyperparameters are provided in \autoref{tab:params_power}, \autoref{tab:params_hepmass}, and \autoref{tab:params_bsds}.
Then, we trained each model until no validation improvement had been observed for 50 epochs.
We then evaluated the model with the best validation score among all candidate models on the test dataset to obtain a log-likelihood score.
We performed this procedure with three separate random seeds, and report the average and standard error across the runs in \autoref{tab:UCI}.   

We searched over all combinations of parameters listed in \autoref{tab:params_power}, \autoref{tab:params_hepmass}, and \autoref{tab:params_bsds}.
For example, on HEPMASS or MINIBOONE, our set of candidate MAF models included: for $\numlayers = 5,$ an autoregressive network of size of either $2 \times 128$, $2 \times 512$, or $2 \times 1024$; for $\numlayers = 10,$ an autoregressive network size of either $2 \times 128$, $2 \times 512$, or $2 \times 1024$; and for $\numlayers = 20,$ again an autoregressive network size of either $2 \times 128$, $2 \times 512$, or $2 \times 1024$; this gave us a total of $9$ candidate MAF models for each seed.
The set of candidate \modelacr-MAF models can similarly be determined via the table and gave us a total of $8$ candidate models for each seed.
We maintained this split of $9$ candidates for MAF and $8$ candidates for \modelacr-MAF across datasets to fairly compare against the baseline by allowing them more configurations.
We also considered deeper and wider MAF models to compensate for the additional parameters introduced by $\NN_\Fwdmap$, $\NN_p$, and $\NN_q$ in the \modelacr-MAFs.
Finally, we allowed the baseline MAF models to use batch normalization between MADE layers as recommended by \citet{papamakarios2017masked}, but we do not use them within \modelacr-MAFs as the structure of our $\Fwdmap$ generalises this transformation.

We should note that our evaluation of models is slightly different from \citet{papamakarios2017masked}.
For the model which scores best on the validation set, \citet{papamakarios2017masked} report the average and standard deviation of log-likelihood across the points in the test dataset.
However, our error bars emerge as the error in average test-set log-likelihood across \emph{multiple} runs of the same experiment; this style of evaluation is often employed in other works as well (e.g.\ FFJORD \citep{grathwohl2018ffjord}, NAF \citep{huang2018neural}, and SOS \citep{jaini2019sum} as noted in \citet[Table~1]{durkan2019neural}).

\subsubsection{Neural Spline Flows} \label{sec:nsf_experiments}

\begin{table}
\caption{\modelacr-NSF configurations for all tabular datasets. The number of hidden features in the autoregressive network is referred to as $n_h$.} \label{tab:params_nsf}
\begin{center}
\begin{small}
\begin{sc}
\begin{tabular}{lllll}
\toprule
  &\multicolumn{1}{c}{$\NN_p$ size}&\multicolumn{1}{c}{$\NN_q$ size}&\multicolumn{1}{c}{$\NN_\Fwdmap$ size}& $n_h$ vs. Baseline \\
\midrule
\modelacr-NSF-1 (MINIBOONE) & $3 \times 50$ & $2 \times 10$ & $3 \times 25$ & fewer\\
\modelacr-NSF-1 (non-MINIBOONE) & $3 \times 200$ & $2 \times 10$ & $3 \times 100$ & fewer\\
\modelacr-NSF-2 & $3 \times 200$ & $2 \times 10$ & $3 \times 100$ & same\\
\bottomrule
\end{tabular}
\end{sc}
\end{small}
\end{center}
\end{table}

The experiment comparing NSF baselines to \modelacr-NSFs mirrors the experimental setup in \citet{durkan2019neural}.
Specifically, we constructed baseline NSFs that exactly copied the settings in \citet[Table~5]{durkan2019neural}.
We also built \modelacr-NSFs using these baseline settings, although for the \modelacr-NSF-1 model we lowered the number of hidden channels in the autoregressive networks so that the total number of trainable parameters matched that of the baseline.
Our parameter settings are provided in \autoref{tab:params_nsf}; note that parameter settings are homogeneous across datasets, besides MINIBOONE for which we reduced the size $\NN_p$ and $\NN_F$ by a factor of $4$ as per \citet{durkan2019neural}\footnote{Indeed, there was no choice of $n_h$ which would allow us to achieve the same number of parameters as the baseline for the models noted in row 2 of \autoref{tab:params_nsf}.}.
We trained both NSFs and \modelacr-NSFs for a number of training epochs corresponding to the number of training steps divided by the number of batches in the training set, i.e.
\[
    n_e = \left\lceil n_s / \left(n_t / n_b \right) \right\rceil,
\]
where $n_e$ is the number of epochs, $n_s$ is the number of training steps, $n_b$ is the batch size, and $n_t$ is the number of training data points.
Note that $n_s$ and $n_b$ are from \citet[Table~5]{durkan2019neural}, and $n_t$ is fixed by the pre-processing steps from \citet{papamakarios2017masked}.
We then evaluated the test-set performance of each model after the pre-specified number of epochs, averaging across three seeds, and put the results in \autoref{tab:UCI}. 
We again average randomness across seeds, rather than across points in the test set, as discussed in the previous section.

We quickly note here that we selected our parameters after trying a few settings on various UCI datasets. 
There were other settings which performed better for individual datasets that are not included here, as we would like the proposed configurations to be as homogeneous as possible.
It appeared as though the NSF models were already fairly good at modelling the data, which allowed us to make $\NN_q$ much smaller while still achieving good inference.

We also should note that we wrapped our code around the NSF bijection code from \url{https://github.com/bayesiains/nsf}.
We also disable weight decay in all of these experiments without observing any problems with convergence.

\subsubsection{Ablating $\fwdmap$} \label{sec:ablating-fwdmap}

\begin{table}
\caption{Mean $\pm$ standard error of average test set log-likelihood (higher is better). Best performing runs are shown in bold. \modelacr-Id-1 had $\scale \equiv 0$ and $\trans = \Id$. \modelacr-Id-2 had $\scale \equiv 0$ and $\trans = \NN_\Fwdmap$. \modelacr-Id-3 had $(\scale, \trans) = \NN_\Fwdmap$.}
\label{tab:UCI-ablation}
\begin{center}
\begin{small}
\begin{sc}
\begin{tabular}{lllll}
\toprule
                                                & Power                     & Gas                       & Hepmass                   & Miniboone\\
\midrule
\modelacr-Id-1 ($\NN_q = 10 \times 2$)   & $0.43 \pm 0.01$           & $10.92 \pm 0.10$          & $-17.06 \pm 0.05$         & $-11.26 \pm 0.03$ \\
\modelacr-Id-1 ($\NN_q = 100 \times 4$)  & $0.42 \pm 0.01$           & $10.86 \pm 0.16$          & $-17.44 \pm 0.09$         & $-10.91 \pm 0.04$ \\
\midrule
\modelacr-Id-2 ($\NN_q = 10 \times 2$)   & $0.45 \pm 0.01$   & $10.43 \pm 0.08$          & $-17.63 \pm 0.10$         & $-11.13 \pm 0.08$ \\
\modelacr-Id-2 ($\NN_q = 100 \times 4$)  & $0.47 \pm 0.01$   & $10.89 \pm 0.18$          & $-17.51 \pm 0.09$         & $-10.75 \pm 0.07$ \\
\midrule
\modelacr-Id-3 ($\NN_q = 10 \times 2$)            & ${\bf 0.50 \pm 0.01}$     & ${\bf 11.32 \pm 0.14}$    & $-17.08 \pm 0.02$         & $-10.45 \pm 0.04$ \\
\modelacr-Id-3 ($\NN_q = 100 \times 4$)           & ${\bf 0.50 \pm 0.01}$     & ${\bf 11.58 \pm 0.12}$    & ${\bf -16.68 \pm 0.07}$   & ${\bf -10.01 \pm 0.04}$ \\
\bottomrule
\end{tabular}
\end{sc}
\end{small}
\end{center}
\end{table}

We ran ablation experiments to gain some insight into the relative importance of $\fwdmap$ in \eqref{eq:our-bijection-family}.
In particular, we considered a 10 layer model ($\numlayers = 10$) where at each layer $\idxvar_\ell$ had the same dimension as the data and $\fwdmap = \Id$ was the identity.
We refer to this model as \modelacr-Id.

We considered three parameterisations of \modelacr-Id.
The first had $\scale \equiv 0$ and $\trans = \Id$, which from our choice \eqref{eq:our-idxcond} of $\idxcond$ corresponds to stacking the following generative process:
\begin{align}
    \noisevar &\sim \priormeas \notag \\
    \epsilon &\sim \Normal(0, \idmat_d) \notag \\
    \datavar &\coloneqq \noisevar - \mean^p(\noisevar) - e^{\varsigma^p(\noisevar)} \odot \epsilon. \label{eq:ablation-generative-x}
\end{align}
Observe this generalise ResFlows, since \eqref{eq:resflow} can be realised by sending $\varsigma^p \to -\infty$ and having $\mean^p < 1$.
Accordingly, we took $\NN_p$ to be a $4 \times 128$ MLP to match the size of the residual blocks used in our tabular ResFlow experiments.

The second \modelacr-Id parameterisation had $\scale \equiv 0$ and $\trans = \NN_\Fwdmap$, which amounts to replacing \eqref{eq:ablation-generative-x} with
\[
    \datavar \coloneqq \noisevar - \trans\left(\mean^p(\noisevar) + e^{\varsigma^p(\noisevar)} \odot \epsilon\right).
\]
To align with the first \modelacr-Id, we took $\NN_\Fwdmap$ and $\NN_p$ to be $2 \times 128$ MLPs, and zeroed out the $\scale$ output of $\NN_\Fwdmap$ to obtain $\scale \equiv 0$.
The third parameterisation had $(\scale, \trans) = \NN_\Fwdmap$, which replaces \eqref{eq:ablation-generative-x} with
\[
    \datavar \coloneqq \exp\left(-\scale\left(\mean^p(\noisevar) + e^{\varsigma^p(\noisevar)} \odot \epsilon\right)\right) \odot \noisevar - \trans\left(\mean^p(\noisevar) + e^{\varsigma^p(\noisevar)} \odot \epsilon\right).
\]
Again, we took $\NN_\Fwdmap$ and $\NN_p$ to be $2 \times 128$ MLPs in this case.

We ran all configurations with two different choices of $\NN_q$: a $2 \times 10$ MLP as in our tabular ResFlow experiments, as well as a $4 \times 100$ MLP.
The results are given in \autoref{tab:UCI-ablation}.\footnote{Due to computational constraints we did not run these experiments on BSDS300.}
Observe that these models performed comparably or better than the 100-layer ResFlows, but worse than the \modelacr-ResFlows and \modelacr-MAFs in \autoref{tab:UCI}.
As discussed in \autoref{sec:advantages-over-nfs}, we conjecture this occurs because a \modelacr-Id requires greater complexity from $\idxcond$ to make up for its simple choice of $\fwdmap$, which in turn makes inference harder and hence the ELBO \eqref{eq:elbo} looser, resulting in a poorer model that is learned overall.
Likewise, note that the best performance in all cases was obtained when $(\scale, \trans) = \NN_\Fwdmap$.
This provides some justification for the generality of our choice of \eqref{eq:our-bijection-family}, as opposed to simpler alternatives that omit $\scale$ or $\trans$.

\subsection{Image Experiments} \label{sec:images}

In all our image experiments we applied the same uniform dequantisation scheme as \citet{theis2015note}, after which we applied the logit transform of  \citet{dinh2016density} with $\alpha = 10^{-5}$ for Fashion-MNIST and $\alpha = 0.05$ for CIFAR10.

\subsubsection{ResFlow}

For our baseline ResFlow experiments we used the same architecture as \citet{chen2019residual}.
In particular, our convolutional residual blocks (denoted Conv-ResBlock) had the form
\[
    \text{LipSwish} \to \text{$3 \times 3$ Conv} \to \text{LipSwish} \to \text{$1 \times 1$ Conv} \to \text{LipSwish} \to \text{$3 \times 3$ Conv},
\]
while our fully connected residual blocks (denoted FC-ResBlock) had the form
\[
    \text{LipSwish} \to \text{Linear} \to \text{LipSwish} \to \text{Linear},
\]
with a residual connection from the input to the output in both cases.
The overall architecture of the flow in all cases was:
\[
    \text{Image} \to \text{LogitTransform}(\alpha) \to \text{$k \times$ \text{Conv-ResBlock}} \to [ \text{Squeeze} \to \text{$k \times$ \text{Conv-ResBlock}}] \times 2 \to \text{$4 \times$ \text{FC-ResBlock}},
\]
where the Squeeze operation was as defined by \citet{dinh2016density}.
Like \citet{chen2019residual}, we used ActNorm layers \citep{kingma2018glow} before and after each residual block.

Due to computational constraints, the models we considered were smaller than those used by \citet{chen2019residual}.
In particular, our smaller ResFlow models used 128 hidden channels in their Conv-ResBlocks, 64 hidden channels in the linear layers of their FC-ResBlocks, and had $k = 4$.
Our larger ResFlow models used 256 hidden channels in their Conv-ResBlocks, 128 hidden channels in the linear layers of their FC-ResBlocks, and had $k = 6$.
In contrast, \citet{chen2019residual} used 512 hidden channels in their Conv-ResBlocks, 128 hidden channels in their FC-ResBlocks, and had $k = 16$.

As described for our tabular experiments, we used the same estimation scheme as \citet{chen2019residual}. Additionally:
\begin{itemize}
    \item We took $\kappa = 0.98$;
    \item We used the Neumann gradient series expression for the log Jacobian \citep[(8)]{chen2019residual} and computed gradients in the forward pass \citep[(9)]{chen2019residual} to reduce memory overhead;
    \item We used an adaptive rather than a fixed number of power iterations for spectral normalisation \citep{gouk2018regularisation}, with a tolerance of 0.001;
\end{itemize}

For the \modelacr-ResFlows, we augmented the smaller baseline ResFlow by treating each composition of $\text{ActNorm} \to \text{ResBlock}$, as well as the final $\text{ActNorm}$, as an instance of $\fwdmap$ in \eqref{eq:our-bijection-family}.
Each $\NN_\Fwdmap$, $\NN_p$, and $\NN_q$ was a ResNet \citep{he2016deep, he2016identity} consisting of $2$ residual blocks with $32$ hidden channels (denoted $2 \times 32$).
We gave each $\idxvar_\ell$ the same shape as a single channel of $\noisevar_\ell$, and upsampled to the dimension of $\noisevar_\ell$ by adding channels at the output of each $\NN_\Fwdmap$.
Note that we did not experiment with using the larger baseline ResFlow model as the basis for a \modelacr.

For all models we used a learning rate of $10^{-3}$ and a batch size of 64.

\autoref{fig:first_ResFlow_sample} through to \autoref{fig:last_ResFlow_sample} show samples synthesised from the ResFlow and \modelacr-ResFlow density models trained on MNIST and CIFAR-10.

\subsubsection{RealNVP}

For our RealNVP-based image experiments, we took the baseline to be a RealNVP with the same architecture used by  \citet{dinh2016density} for their CIFAR-10 experiments.
In particular, we used 10 affine coupling layers with the corresponding alternating channelwise and checkerboard masks.
Each coupling layer used a ResNet \citep{he2016deep, he2016identity} consisting of $8$ residual blocks of $64$ channels (denoted $8 \times 64$).
We replicated the multi-scale architecture of \citet{dinh2016density}, squeezing the channel dimension after the first 3 coupling layers, and splitting off half the dimensions after the first 6.
This model had 5.94M parameters for Fashion-MNIST and 6.01M parameters for CIFAR-10.

For the \modelacr-RealNVP, we considered each affine coupling layer to be an instance of $\fwdmap$ in \eqref{eq:our-bijection-family}.
When choosing the size of our networks, we sought to maintain roughly the same depth over which gradients were propagated as in the baseline.
To this end, our coupling networks were $4 \times 64$ ResNets, each $\NN_p$ and $\NN_q$ were $2 \times 64$ ResNets, and each $\NN_\Fwdmap$ was a $2 \times 8$ ResNet.
We gave each $\idxvar_\ell$ the same shape as a single channel of $\noisevar_\ell$, and upsampled to the dimension of $\noisevar_\ell$ by adding channels at the output of $\NN_\Fwdmap$.
Our model had 5.99M parameters for Fashion-MNIST and 6.07M parameters for CIFAR-10.

For completeness, we also trained a RealNVP model with coupler networks of size $4 \times 64$ to match our \modelacr-RealNVP configuration.
This model had 2.99M parameters for Fashion-MNIST and 3.05M for CIFAR-10.

In all cases for these experiments we used a learning rate of $10^{-4}$ and a batch size of 100.

\autoref{fig:first_RealNVP_sample} through to \autoref{fig:last_RealNVP_sample} show samples synthesised from the RealNVP and \modelacr-RealNVP density models trained on Fashion-MNIST and CIFAR-10.

\subsection{2-D Experiments} \label{sec:2d-experiments}

To gain intuition about our model, we ran experiments on some simple 2-D datasets.
For the datasets in \autoref{fig:iResNet-2uniforms-annulus}, we used a 10-layer ResFlow, a 100-layer ResFlow, and 10-layer \modelacr-ResFlow.
For the \modelacr-ResFlows we took $d_\idxspace = 1$.
Other architectural and training details were the same as for the tabular experiments described in \autoref{sec:uci-experiments} and \autoref{sec:residual-flows-experiments-supp}.
The resulting average test set log-likelihoods for the top dataset were:
\begin{itemize}
    \item -1.501 for the 10-layer ResFlow
    \item -1.419 for the 100-layer ResFlow
    \item -1.409 for the 10-layer \modelacr-ResFlow
\end{itemize}
The final average test set log likelihoods for the bottom dataset were:
\begin{itemize}
    \item -2.357 for the 10-layer ResFlow
    \item -2.287 for the 100-layer ResFlow
    \item -2.275 for the 10-layer \modelacr-ResFlow
\end{itemize}
Note that in both cases the \modelacr-ResFlow slightly outperformed the 100-layer ResFlow.

We additionally ran several experiments comparing a baseline MAF against a \modelacr-MAF on the 2-D datasets shown in \autoref{fig:misc_2d}.
The baseline MAFs had 20 autoregressive layers, while the \modelacr-MAFs had 5.
The network used at each layer had 4 hidden layers of 50 hidden units (denoted $4 \times 50$).
For the \modelacr-MAF, we took $d_\idxspace = 1$, and used $2 \times 10$ MLPs for $\NN_\Fwdmap$ and $4 \times 50$ MLPs for $\NN_p$ and $\NN_q$.
In total the baseline MAF had 160160 parameters, while our model had 119910 parameters.

The results of these experiments are shown in \autoref{fig:misc_2d}.
Observe that \modelacr-MAF consistently produces a more faithful representation of the target distribution than the baseline, and in all cases achieved higher average test set log probability.
A failure mode of our approach is exhibited in the spiral dataset, where our model still lacks the power to fully capture the topology of the target.
However, we did not find it difficult to improve on this: by increasing the size of $\NN_p$ to $8 \times 50$ (and keeping all other parameters fixed), we were able to obtain the result shown in \autoref{fig:2spirals_big}.
This model had a total of 221910 parameters.
We also tried a larger MAF model with autoregressive networks of size $8 \times 50$, (obtaining 364160 parameters total).
This model diverged after approximately 160 epochs.
The result after 150 epochs is shown in \autoref{fig:2spirals_big}.

\begin{figure}[h]
    \centering
    \includegraphics{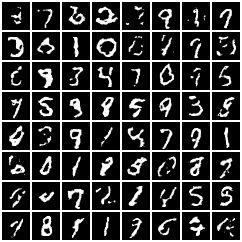}
    \caption{Synthetic MNIST samples generated by the small baseline ResFlow model}
    \label{fig:first_ResFlow_sample}
\end{figure}

\begin{figure}[h]
    \centering
    \includegraphics{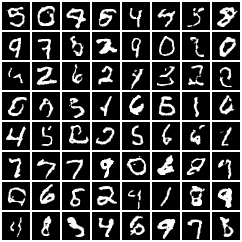}
    \caption{Synthetic MNIST samples generated by the large baseline ResFlow model}
\end{figure}

\begin{figure}[h]
    \centering
    \includegraphics{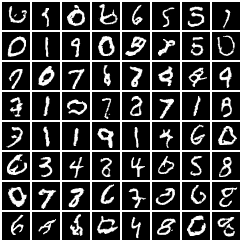}
    \caption{Synthetic MNIST samples generated by the \modelacr-ResFlow model}
\end{figure}

\begin{figure}[h]
    \centering
    \includegraphics{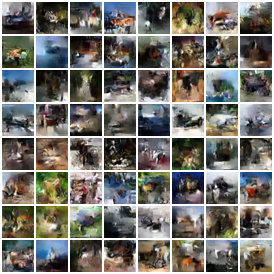}
    \caption{Synthetic CIFAR-10 samples generated by the small baseline ResFlow model}
\end{figure}

\begin{figure}[h]
    \centering
    \includegraphics{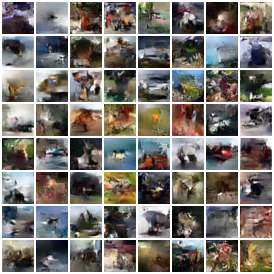}
    \caption{Synthetic CIFAR-10 samples generated by the large baseline ResFlow model}
\end{figure}

\begin{figure}[h]
    \centering
    \includegraphics{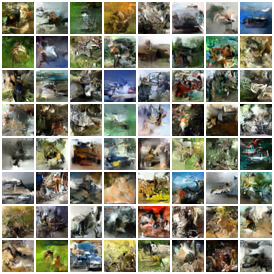}
    \caption{Synthetic CIFAR-10 samples generated by the \modelacr-ResFlow model}
    \label{fig:last_ResFlow_sample}
\end{figure}

\begin{figure}[h]
    \centering
    \includegraphics{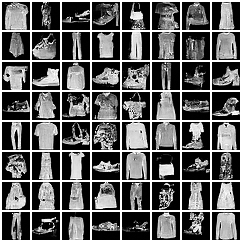}
    \caption{Synthetic Fashion-MNIST samples generated by RealNVP with coupling networks of size $4 \times 64$}
    \label{fig:first_RealNVP_sample}
\end{figure}

\begin{figure}[h]
    \centering
    \includegraphics{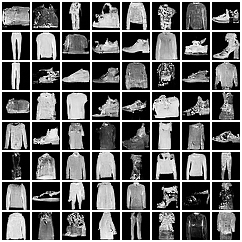}
    \caption{Synthetic Fashion-MNIST samples generated by RealNVP with coupling networks of size $8 \times 64$}
\end{figure}

\begin{figure}[h]
    \centering
    \includegraphics{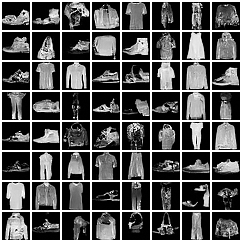}
    \caption{Synthetic Fashion-MNIST samples generated by \modelacr-RealNVP}
\end{figure}

\begin{figure}[h]
    \centering
    \includegraphics{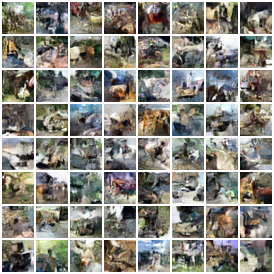}
    \caption{Synthetic CIFAR-10 samples generated by RealNVP with coupling networks of size $4 \times 64$}
\end{figure}

\begin{figure}[h]
    \centering
    \includegraphics{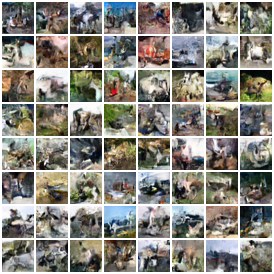}
    \caption{Synthetic CIFAR-10 samples generated by RealNVP with coupling networks of size $8 \times 64$}
\end{figure}

\begin{figure}[h]
    \centering
    \includegraphics{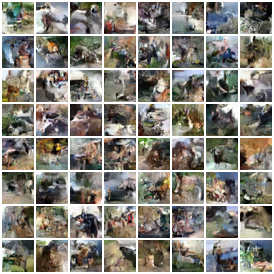}
    \caption{Synthetic CIFAR-10 samples generated by \modelacr-RealNVP}
    \label{fig:last_RealNVP_sample}
\end{figure}

\begin{figure}[h]
    \centering
    \begin{subfigure}{.5\textwidth}
        \centering
        \includegraphics[width=\textwidth]{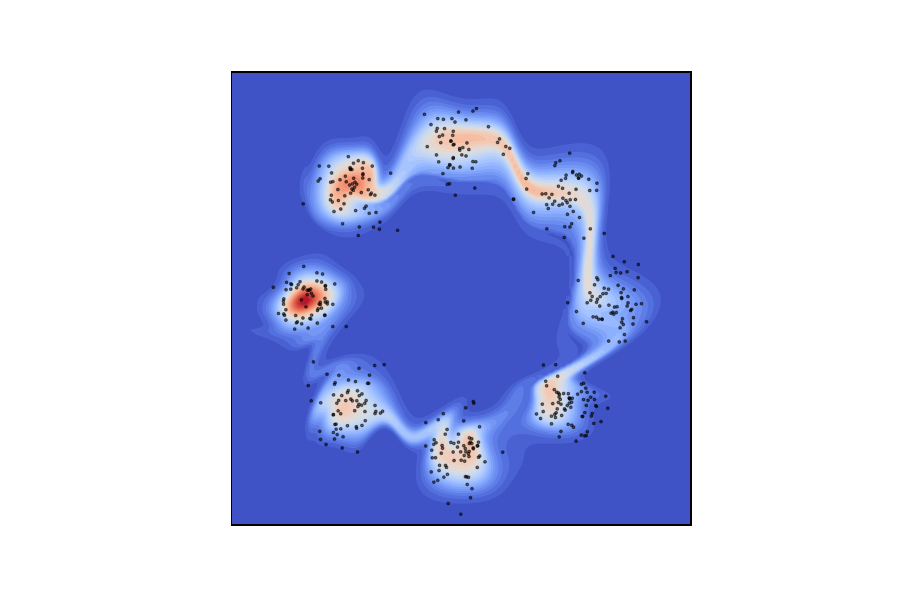}
    \end{subfigure}%
    \begin{subfigure}{.5\textwidth}
      \centering
        \includegraphics[width=\textwidth]{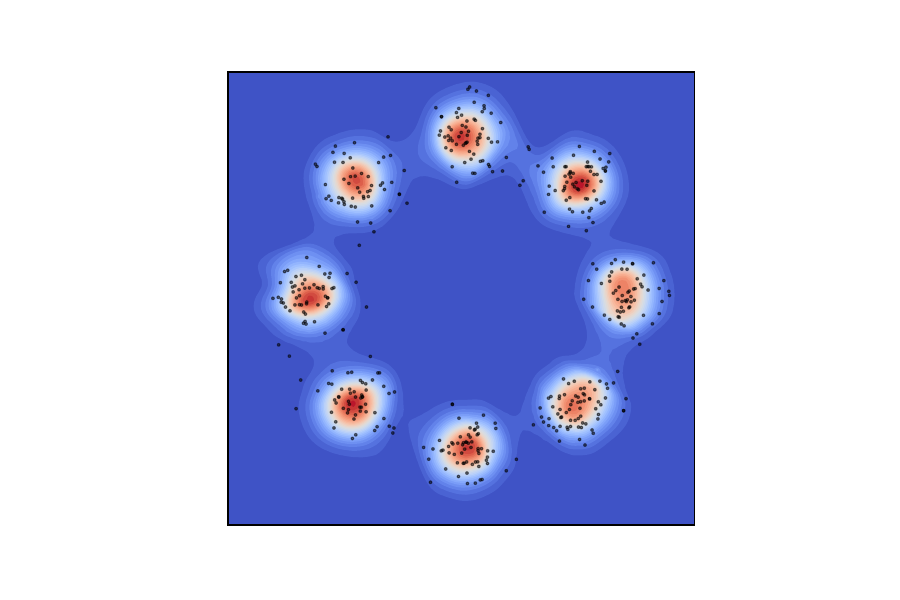}
    \end{subfigure}
    
    \begin{subfigure}{.5\textwidth}
        \centering
        \includegraphics[width=\textwidth]{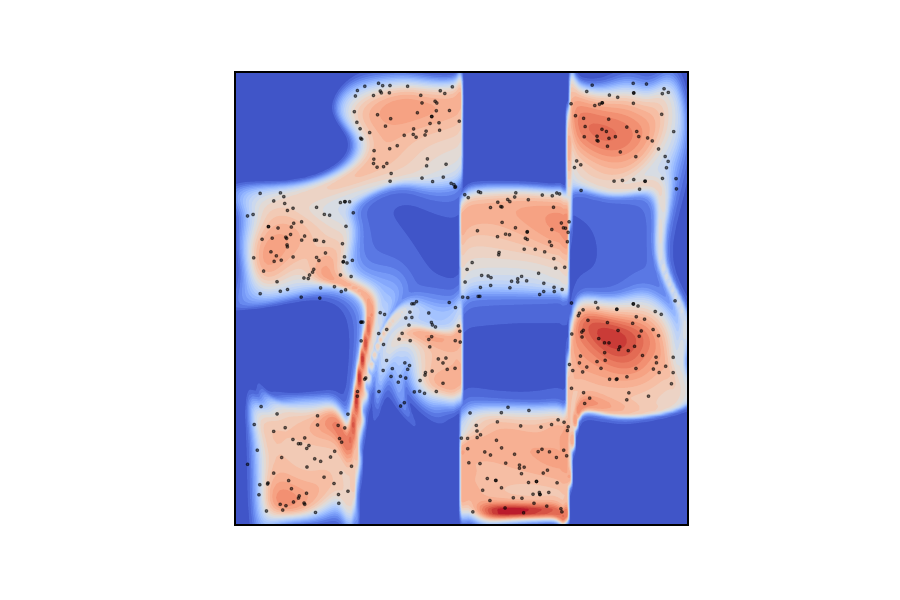}
    \end{subfigure}%
    \begin{subfigure}{.5\textwidth}
      \centering
        \includegraphics[width=\textwidth]{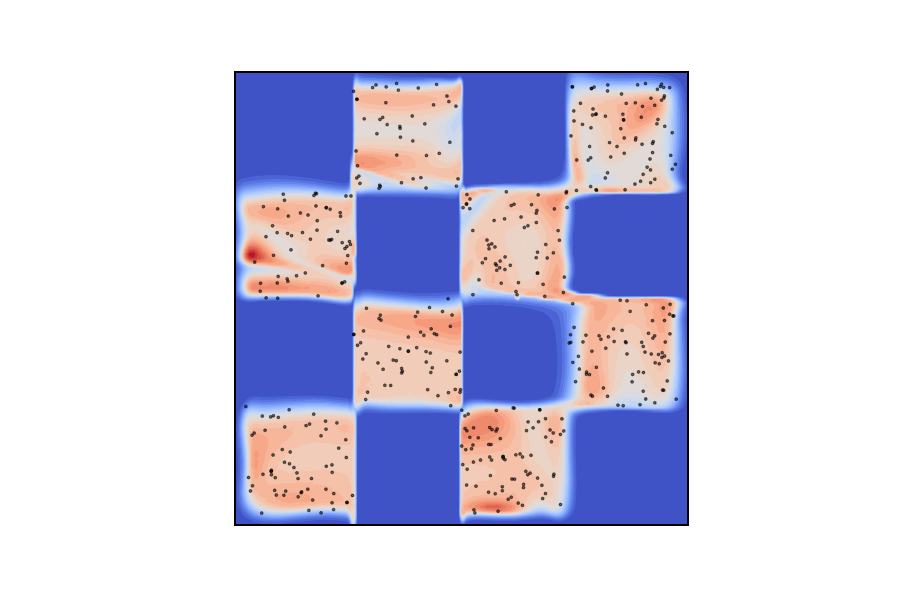}
    \end{subfigure}
    
    \begin{subfigure}{.5\textwidth}
        \centering
        \includegraphics[width=\textwidth]{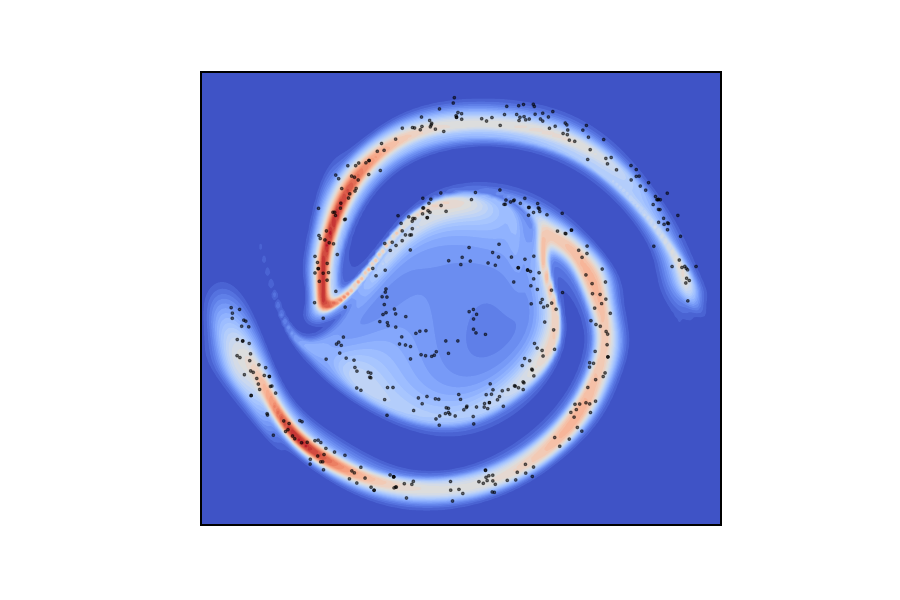}
    \end{subfigure}%
    \begin{subfigure}{.5\textwidth}
      \centering
        \includegraphics[width=\textwidth]{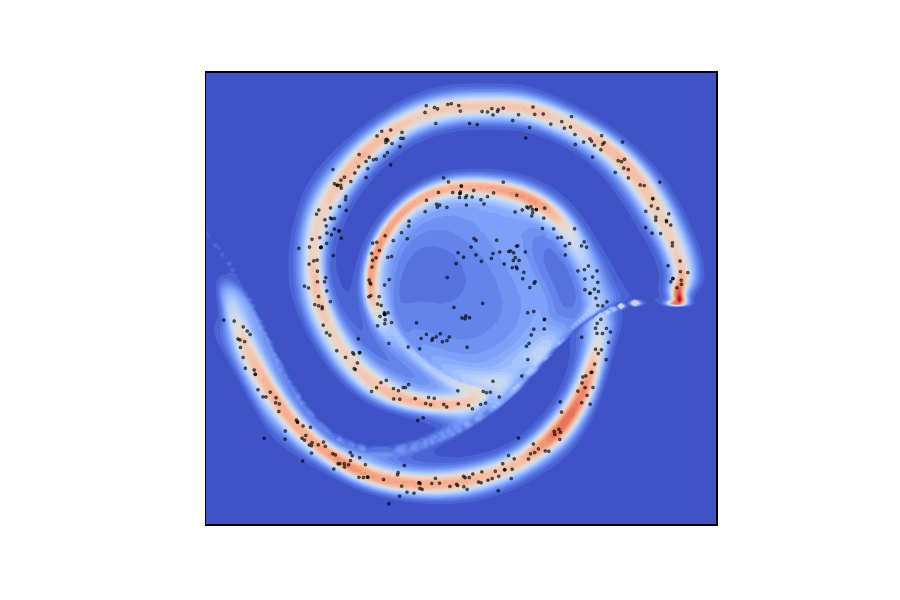}
    \end{subfigure}
    
    \caption{Density models learned by a standard 20 layer MAF (left) and by a 5 layer \modelacr-MAF (right) for a variety of 2-D target distributions.  Samples from the target are shown in black.}
    \label{fig:misc_2d}
\end{figure}

\begin{figure}[h]
    \centering
    \begin{subfigure}{.5\textwidth}
      \centering
        \includegraphics[width=\textwidth]{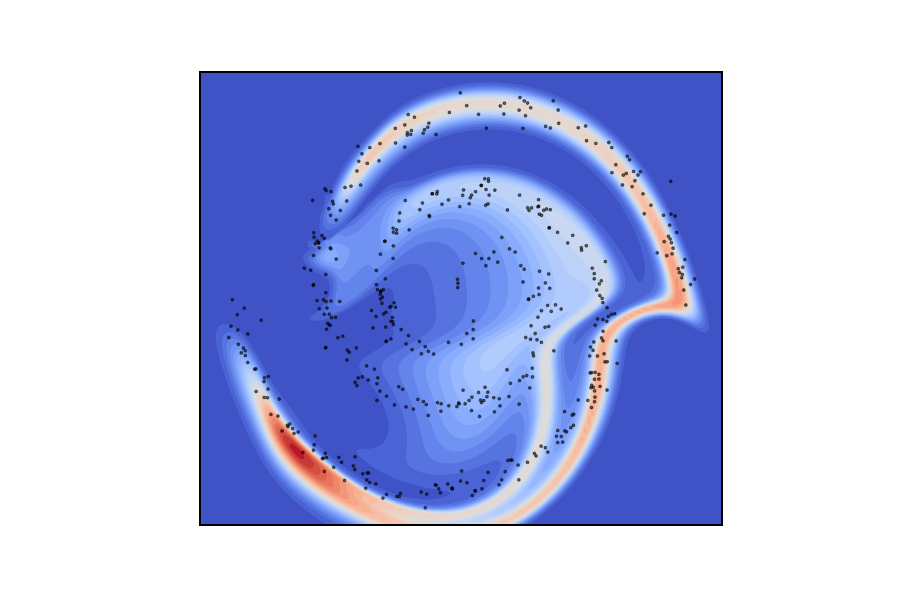}
    \end{subfigure}%
    \begin{subfigure}{.5\textwidth}
        \centering
        \includegraphics[width=\textwidth]{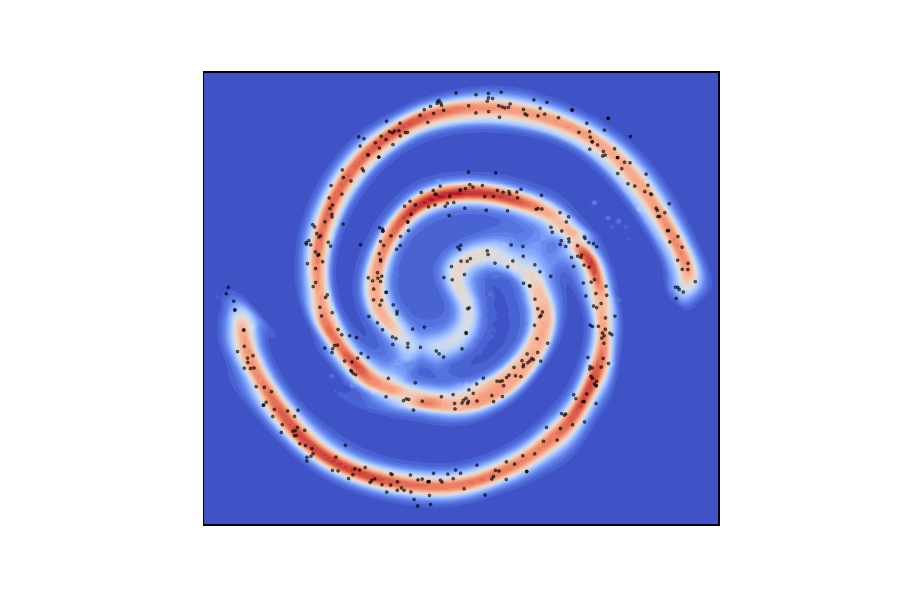}
    \end{subfigure}

    \caption{Density models learned by a larger 20 layer MAF (left) and a larger 5 layer \modelacr-MAF (right) for the spirals dataset.}
    \label{fig:2spirals_big}
\end{figure}

\end{document}